\documentclass{article}

\usepackage[margin=1in]{geometry} %
\usepackage{titling}              %
\usepackage{authblk}              %
\pretitle{\vspace{-1em}\noindent\rule{\linewidth}{0.4pt}\par\nointerlineskip
          \vskip 0.6em\begin{flushleft}\LARGE\bfseries}
\posttitle{\par\end{flushleft}\vskip 0.5em}
\preauthor{\begin{flushleft}\large}
\postauthor{\par\end{flushleft}}
\predate{\begin{flushleft}\small}
\postdate{\par\end{flushleft}\nointerlineskip\noindent\rule{\linewidth}{0.4pt}}
\renewcommand\Affilfont{\normalsize}
\makeatletter
\renewcommand\AB@affilsepx{, \protect\Affilfont}
\makeatother

\usepackage[backend=biber, style=numeric, natbib=true, sorting=nyt, sortcites=true,
            maxbibnames=20, minbibnames=10, maxcitenames=2, mincitenames=1,
            giveninits=false]{biblatex}
\DeclareNameAlias{default}{given-family}
\DeclareFieldFormat{url}{\url{#1}}
\AtEveryBibitem{%
  \clearfield{eprintclass}%
  \iffieldundef{eprint}
    {\iffieldundef{doi}{}{\clearfield{url}}}
    {\clearfield{doi}\clearfield{url}}}
\makeatletter

\renewcommand{\normalsize}{%
  \@setfontsize\normalsize\@xpt\@xipt
  \abovedisplayskip      7\p@ \@plus 2\p@ \@minus 5\p@
  \abovedisplayshortskip \z@ \@plus 3\p@
  \belowdisplayskip      \abovedisplayskip
  \belowdisplayshortskip 4\p@ \@plus 3\p@ \@minus 3\p@
}
\normalsize

\def\@listi  {\leftmargin\leftmargini}
\def\@listii {\leftmargin\leftmarginii
              \labelwidth\leftmarginii
              \advance\labelwidth-\labelsep
              \topsep  2\p@ \@plus 1\p@    \@minus 0.5\p@
              \parsep  1\p@ \@plus 0.5\p@ \@minus 0.5\p@
              \itemsep \parsep}
\def\@listiii{\leftmargin\leftmarginiii
              \labelwidth\leftmarginiii
              \advance\labelwidth-\labelsep
              \topsep    1\p@ \@plus 0.5\p@ \@minus 0.5\p@
              \parsep    \z@
              \partopsep 0.5\p@ \@plus 0\p@ \@minus 0.5\p@
              \itemsep \topsep}
\def\@listiv {\leftmargin\leftmarginiv
              \labelwidth\leftmarginiv
              \advance\labelwidth-\labelsep}
\def\@listv  {\leftmargin\leftmarginv
              \labelwidth\leftmarginv
              \advance\labelwidth-\labelsep}
\def\@listvi {\leftmargin\leftmarginvi
              \labelwidth\leftmarginvi
              \advance\labelwidth-\labelsep}

\newlength{\@capskipa}
\newlength{\@capskipb}
\renewenvironment{table}
  {\setlength{\abovecaptionskip}{\@capskipb}%
   \setlength{\belowcaptionskip}{\@capskipa}%
   \@float{table}}
  {\end@float}

\patchcmd{\@starttoc}{\makeatletter}{\makeatletter\parskip\z@}{}
  {\PackageWarning{custom_style}{Could not patch \string\@starttoc}}

\renewcommand{\footnoterule}{\kern-3\p@ \hrule width 12pc \kern 2.6\p@}

\makeatother

\usepackage[T1]{fontenc}    %
\usepackage[colorlinks=true, linkcolor=blue, citecolor=blue, urlcolor=blue, hypertexnames=false]{hyperref}       %
\usepackage{url}            %
\usepackage{booktabs}       %
\usepackage{amsfonts}       %
\usepackage{nicefrac}       %
\usepackage{microtype}      %
\usepackage{xcolor}         %

\usepackage{amsmath,amsfonts,bm}

\def\1{\bm{1}}

\def\vv{{\bm{v}}}

\def\mA{{\bm{A}}}
\def\mB{{\bm{B}}}
\def\mC{{\bm{C}}}

\def\mG{{\bm{G}}}

\def\mI{{\bm{I}}}

\def\mM{{\bm{M}}}

\def\mO{{\bm{O}}}

\def\mQ{{\bm{Q}}}
\def\mR{{\bm{R}}}

\def\mU{{\bm{U}}}
\def\mV{{\bm{V}}}
\def\mW{{\bm{W}}}
\def\mX{{\bm{X}}}

\def\mZ{{\bm{Z}}}

\def\mSigma{{\bm{\Sigma}}}

\DeclareMathAlphabet{\mathsfit}{\encodingdefault}{\sfdefault}{m}{sl}
\SetMathAlphabet{\mathsfit}{bold}{\encodingdefault}{\sfdefault}{bx}{n}

\newcommand{\R}{\mathbb{R}}

\usepackage{graphicx}
\graphicspath{{figures/}}

\usepackage{enumitem}

\usepackage{algorithm}
\usepackage{algpseudocode}

\usepackage[page, header, toc]{appendix}

\usepackage{amsthm,amssymb}

\usepackage[capitalize,nameinlink]{cleveref}

\newtheorem{theorem}{Theorem}

\newtheorem{definition}[theorem]{Definition}

\title{Dion3: Full-stack orthogonal updates}

\DeclareRobustCommand{\msrnote}{\footnotemark[1]}
\author[1]{Noah Amsel\thanks{Work performed while at Microsoft Research.}}
\author[2]{Jack Zhang}
\author[3]{Kwangjun Ahn\msrnote}
\author[4]{Ali Naeimi}
\author[5]{Austin Feng\msrnote}
\author[2]{Berlin Chen}
\author[2]{Tri Dao}
\author[6]{John Langford}
\affil[1]{New York University}
\affil[2]{Princeton University}
\affil[3]{NVIDIA}
\affil[4]{Independent Researcher}
\affil[5]{Yale University}
\affil[6]{Microsoft Research}

\begin{document}
\raggedbottom

\maketitle

\begin{abstract}
The Muon optimizer incurs a significant overhead cost due to its cubic-time Newton-Schulz orthogonalization step.
When weights are sharded, communication overhead compounds this computational cost, eroding the benefits of Muon in many settings.
We present Dion3, a revision of Muon that targets this overhead at every level of the stack.
Our Gram Newton-Schulz algorithm reduces the FLOP cost of orthogonalization, our CuteDSL kernels accelerate it by exploiting symmetry, and our megabatching strategy reduces communication overhead.
Moreover, we propose a simple change to the update rule that cuts costs even further: selecting only a fraction of the momentum matrix's rows to orthogonalize at each step.
This update rule improves on Dion (another ``compressed'' version of Muon), in both speed and performance \citep{ahn2025dion}.
Overall, Dion3 matches or improves on the loss achieved by Muon but reduces optimizer step time by up to $6\times$.
Dion3 is available via the \texttt{dion} package\footnote{\url{https://github.com/microsoft/dion}} as a drop-in replacement for Muon.

\end{abstract}

\begin{figure}[ht!]
\centering
\includegraphics[width=0.75\textwidth]{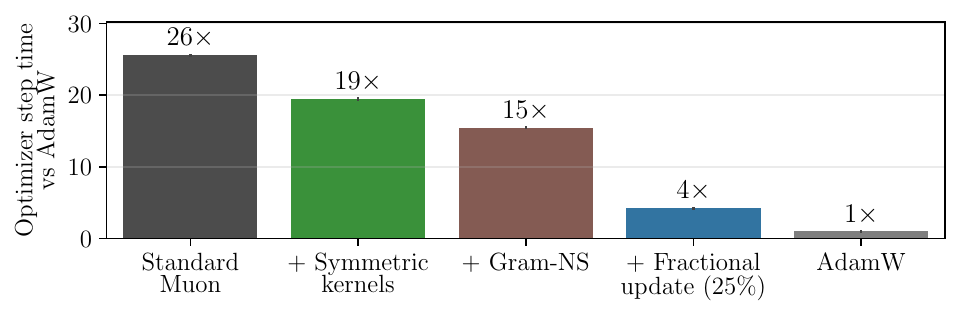}
\caption{Optimizer step time of Muon (excluding forward/backward pass) relative to AdamW for a 7B-parameter language model trained on four GH200s.
Each colored bar adds one of our contributions on top of the previous one.
Together they reduce Muon's cost from $26\times$ AdamW to just $4\times$.}
\label{fig:optstep-bar}
\end{figure}

\newpage
\setcounter{tocdepth}{2}
\tableofcontents
\newpage

\section{Introduction}
Muon is becoming the method of choice for training frontier LLMs like Kimi K2 and GLM-5 \citep{kimiteam2026kimik2openagentic,glm5team2026glm5vibecodingagentic}.
Compared to AdamW, Muon needs fewer optimizer steps to reach a given loss, but each step is more expensive.
This overhead is due to Muon's Newton-Schulz orthogonalization procedure, a cubic-time matrix operation not present in older optimizers.
As model size increases, the overhead of computing each Muon step grows rapidly.
Moreover, Muon is more difficult than traditional optimizers to parallelize, introducing additional communication overhead in distributed settings.
Thus, scaling Muon presents a range of algorithmic and systems-level challenges that diminish its effectiveness in the most demanding settings and hinder its adoption to new problems.

This paper gives a comprehensive answer to these challenges.
We provide a full-stack solution that performs well across a wide range of model sizes, architectures, cluster sizes, and parallelism strategies.
We package our contributions into a single optimizer, Dion3, built on four improvements that each reduce the cost of Muon's orthogonalization step and compound when used together:
\begin{enumerate}
  \item \textbf{Gram Newton-Schulz}, a mathematically equivalent reformulation of Newton-Schulz that iterates on the small symmetric Gram matrix and cuts the FLOP cost of each orthogonalization dramatically (\Cref{sec:gram-ns}).
  \item Custom \textbf{GPU kernels for symmetric matrix multiplication} written in CuteDSL, which speed up Gram Newton-Schulz and take maximal advantage of its symmetric structure (\Cref{sec:kernels}).
  \item A new \textbf{optimizer update rule} that subsamples the rows or columns of the momentum matrix before orthogonalizing it to make each step faster (\Cref{sec:nordion2}).
  Our update rule is simpler and faster than Dion \citep{ahn2025dion} while matching or improving the optimization quality of Muon and NorMuon.
  \item \textbf{Megabatched communication}, which eliminates a significant source of overhead by reducing the number of rounds of communication per optimizer step to a small constant.
\end{enumerate}

Our improvements combine to accelerate the optimizer step at every level of the stack.
The GPU kernels at the lowest level speed up Gram Newton-Schulz, which in turn reduces the FLOP cost of executing our optimizer's update rule.
At the top level, the compression achieved by our update rule joins with the megabatch strategy to limit the communication cost.
Each of our contributions is of independent value, but combining them gives Dion3 maximum flexibility to handle different parallelism strategies, architectures, and model sizes efficiently.
Moreover, Dion3 benefits from synergies between them that further reduce computational cost.
Gram Newton-Schulz uses more symmetric multiplications than standard Newton-Schulz does, compounding the benefits of our CuteDSL kernels.
Likewise, Gram Newton-Schulz is especially fast for matrices with highly asymmetric shapes---like those produced by our update rule's subsampling strategy.
Surprisingly, Dion3 appears to benefit the update \emph{quality} as well as the computational footprint in some settings, even with a strong optimized baseline.

To realize these gains in practice, we provide open-source implementations in the form of two interoperable packages: \texttt{gram-newton-schulz}\footnote{\url{https://github.com/Dao-AILab/gram-newton-schulz}} is a drop-in replacement for Muon's Newton-Schulz routine that also incorporates our CuteDSL kernels, while the \texttt{dion} package implements our optimizer update rule (along with baselines like Muon and NorMuon), and handles parallelism in the distributed setting (megabatching, FSDP, DDP, etc.).
In all, Dion3 allows practitioners working in a wide range of settings to realize the benefits of Muon and related optimizers with almost no overhead.

\section{The Challenge of Scaling Muon}
\subsection{Muon and NorMuon Recap}
\label{sec:muon-recap}

The Muon optimizer \citep{jordan2024muon} is best described as steepest-direction descent with respect to the spectral norm \citep{bernstein2025deriving}.
At a given training step, let $\mW \in \mathbb{R}^{n \times m}$ be a weight matrix and let $\mG$ be the gradient of the loss with respect to $\mW$.
The Muon update rule is
\begin{align}
\begin{split}\label{eq:muon-update}
\mM &\gets \mu \mM + \mG \\
\mW &\gets \mW - \eta \operatorname{polar}(\mM)
\end{split}
\end{align}
where $\mu$ is the momentum coefficient, $\eta$ is the learning rate, and $\mM$ is the momentum matrix (with $\mM_0 := \bm{0}$).
In many ways, Muon resembles basic stochastic gradient descent (SGD) with momentum.
Its key innovation is the $\operatorname{polar}$ operation, which is defined as follows:

\begin{definition}[Polar Decomposition]
If $\mX = \mU \mathbf \Sigma \mV^\top$ is the singular value decomposition (SVD) of a matrix, then $\operatorname{polar}(\mX) = \mU \mV^\top$.
\end{definition}
This operation---the ``orthogonalization step''---balances the spectrum of the update, ensuring that it has full numerical rank.

NorMuon \citep{li2025normuonmakingmuonefficient} is a popular variant of Muon that aims to combine its strengths with those of Adam.
While Muon's update has a balanced spectrum, it does not have balanced row norms; each step could greatly affect some neurons while almost neglecting others.
Following Adam, NorMuon uses the second moment of the (orthogonalized) gradients to adaptively adjust the stepsize for each neuron, helping all neurons to learn at every step.
A final rescaling ensures that the overall magnitude of the update (as measured in the Frobenius norm) matches that of Muon:
\begin{align}\label{eq:normuon-update}
\begin{split}
\mM &\gets \mu \mM + \mG, \qquad
\mO \gets \operatorname{polar}(\mM)\\
\vv_i &\gets \beta_2 \vv_i + (1 - \beta_2) \cdot \frac1m \sum\nolimits_j \mO_{ij}^2,\qquad
\widehat{\mO}_{ij} \gets \frac{\mO_{ij}}{\sqrt{\vv_i} + \epsilon}\\
\mW &\gets \mW - \eta \frac{\|\mO\|_{\mathsf F}}{\|\widehat{\mO}\|_{\mathsf F}}\widehat{\mO}
\end{split}
\end{align}
The cost of NorMuon's extra normalization steps is negligible compared to that of computing $\operatorname{polar}(\mM)$, so NorMuon's benefits come almost for free.
Our methods, which speed up the computation of $\operatorname{polar}(\mM)$, apply to both Muon and NorMuon alike.

\subsection{Standard Newton-Schulz}
The main challenge to scaling Muon is the need to compute $\operatorname{polar}(\mM)$ for each weight matrix at every step.
Since $\operatorname{polar}(\cdot)$ is expensive to compute exactly, Muon uses the Newton-Schulz method to approximate it.
Newton-Schulz is an iterative method based on matrix polynomials.
Beginning with $\mX_0$, each iteration improves the approximation $\mX_t \approx \operatorname{polar}(\mX_0)$ according to the update rule
\[
\mX_{t+1} = a_t \mX_t + b_t \mX_t \mX_t^\top \mX_t + c_t \left(\mX_t \mX_t^\top\right)^2 \mX_t.
\]
We can interpret Newton-Schulz by understanding how it affects the singular value decomposition.
Let $\mX_0 = \mU \mathbf \Sigma \mV^\top$ be the SVD, where $\mU^\top \mU = \mV^\top \mV = \mI$ and $\mathbf \Sigma$ is diagonal with positive entries called the singular values.
A direct computation shows
\[
\mX_1 = \mU \left(a_1 \mathbf \Sigma + b_1 \mathbf \Sigma^3 + c_1 \mathbf \Sigma^5 \right) \mV^\top = \mU p_1(\mathbf \Sigma) \mV^\top
\]
where $p_1(x) := a_1 x + b_1 x^3 + c_1 x^5$.
Since $\mU$ and $\mV$ have orthonormal columns and $p_1(\mathbf \Sigma)$ is diagonal, the right-hand side of this equation must be the SVD of $\mX_1$. %
By extension, $\mX_T$ also has the same singular vectors $\mU$ and $\mV$ as $\mX_0$, and its singular values have been transformed according to the composition of polynomials $(p_T \circ \cdots \circ p_1)(\mathbf \Sigma)$.
If we normalize the input matrix $\mX_0 = \mX / \|\mX\|_{\mathsf F}$, then all singular values of $\mX_0$ lie in $[0, 1]$.
\citet{jordan2024muon} identified a sequence of degree-5 odd polynomials for which $(p_T \circ \cdots \circ p_1)(x) \approx 1$ on $[0, 1]$.
Therefore,
\[ \mX_T = \mU(p_T \circ \cdots \circ p_1)(\mathbf \Sigma)\mV^\top \approx \mU \mV^\top =: \operatorname{polar}(\mX_0)\]

\begin{algorithm}[H]
\caption{Standard Newton-Schulz}
\label{alg:standard-ns}
\begin{algorithmic}[1]
\Require $\mX \in \mathbb{R}^{n \times m}$ with $n \leq m$, coefficients $\{(a_t, b_t, c_t)\}_{t=1}^5$
\State $\mX \gets \mX \,/\, (\|\mX\|_{\mathsf F} + \epsilon)$ \Comment{Normalize sing.\ vals.\ to $[0, 1]$; $\epsilon = 10^{-7}$}
\State $\mX \gets \texttt{bfloat16}(\mX)$ \Comment{Cast to half precision for speed}
\For{$t = 1, \ldots, 5$} \Comment{Apply $p_t(\mX)$}
    \State $\mA \gets \mX\mX^\top$
    \State $\mB \gets b_t \mA + c_t \mA^2$
    \State $\mX \gets a_t \mX + \mB \mX$
\EndFor
\State \Return $\mX$
\end{algorithmic}
\end{algorithm}

\Cref{alg:standard-ns} gives the standard implementation of Newton-Schulz.
We now analyze its runtime in FLOPs to help us understand its performance bottlenecks.
We count only the cubic-time matrix multiplication operations, ignoring the lower-order scalar multiplications and matrix additions.
For clarity, we let $T$ denote the number of iterations, remembering that Muon typically sets $T=5$.
We also assume without loss of generality that $n \leq m$ and define the aspect ratio $\alpha = m / n \geq 1$.
Intuitively, $\alpha$ measures how asymmetric the shape of the matrix is, with $\alpha = 1$ being square and $\alpha \gg 1$ being very asymmetric.

Each iteration has three steps.
Each step contains a single matrix multiplication ($\mX \mX^\top$, $\mA^2$, $\mB \mX$) costing, respectively, $2mn^2$, $2n^3$, and $2mn^2$ FLOPs for a total cost of $T(4mn^2 + 2n^3) = 2T(2\alpha + 1)n^3$ FLOPs.
When $T=5$, the cost is $(20\alpha + 10)n^3$ spread across 15 matrix-matrix multiplications (GEMMs).

\subsection{Scalability of Muon}
\label{sec:scalability}
While Muon comfortably outperforms traditional optimizers like Adam on a per-step basis, two factors immediately make it more difficult to scale:

\begin{itemize}[leftmargin=*, labelindent=1em]
\item \textbf{Super-linear complexity.}
SGD and Adam perform only inexpensive element-wise operations like addition and scalar multiplication; therefore, the computational complexity of these methods scales linearly with the number of parameters.
Orthogonalization is substantially more expensive; for an $n \times n$ weight matrix, it requires $O(n^3)$ time.
Although sparse MoE architectures keep most weight matrices smaller, they also reduce overall model FLOPs, increasing the relative cost of the optimizer step \citep{essential2025muon} compared to the forward and backward passes. 
Moreover, MoEs still include large weight matrices in their dense layers \citep{liu2024deepseek}.

\item \textbf{Distributed training.}
There are additional challenges when weights are sharded across GPUs, as is standard in distributed training.
Element-wise optimizers like Adam can update each shard separately, but Muon must gather the shards together to compute $\operatorname{polar}$ on the full matrix.
Early implementations of Muon duplicated the orthogonalization step across devices, but this increases its already considerable cost.
\citet{essential2025muon} propose using all-to-all communication along the sharding dimension, with each device processing a different weight in parallel.
\citet{ahn2025dion} implemented this strategy in PyTorch FSDP2, and \citet{lim2025motif2} examined its compute-communication overlap characteristics.
While this approach makes the overhead of distributed Newton-Schulz manageable in some settings, it does not fully resolve the scalability limitations.
\end{itemize}

In \Cref{sec:nordion2}, we introduce a variant of Muon that addresses these two challenges by subselecting (that is, compressing) the momentum matrix before orthogonalizing it.
\Cref{sec:megabatch} describes a megabatching strategy and a flexible, pip-installable package that gracefully handle the distributed setting.
Our other contributions are inspired by the analysis in the previous section, which reveals two shortcomings of the standard Newton-Schulz algorithm:

\begin{itemize}[leftmargin=*, labelindent=1em]
\item \textbf{Symmetric Matrix Multiplication.}
The matrices $\mA = \mX \mX^\top$ and $\mB = b_t \mA + c_t \mA^2$ computed at each iteration of Newton-Schulz are symmetric by definition, but standard Newton-Schulz does not exploit this structure.
We can compute the lower triangular part of these matrices in the usual way and simply copy the results to the upper triangular part.
This technique halves the cost of computing $\mX \mX^\top$ and $\mA^2$, giving an overall total of $T(3\alpha + 1)n^3$ FLOPs.
\Cref{sec:kernels} describes custom CuteDSL kernels that implement this technique.

\item \textbf{Dependence on Aspect Ratio.}
Newton-Schulz's runtime is dominated by the large rectangular matrix multiplications needed to compute $\mX \mX^\top$ and $\mB \mX$, which together cost $3\alpha n^3$ FLOPs per iteration even when using symmetric matrix multiplications.
A typical implementation with $T=5$ requires 10 of these expensive rectangular multiplications.
This strong dependence on $\alpha$ is unfortunate.
Most of the weight matrices in transformer architectures are rectangular, including the MLP weights, MoE weights, and attention projection weights when using GQA or MLA.
Furthermore, we observe that the latest MoE architectures are trending towards finer-grained, sparser experts, meaning that the aspect ratios of their hidden dimensions to intermediate dimensions are increasing as well \citep{kimiteam2026kimik2openagentic,guo2026sonicmoeacceleratingmoeio,yang2025qwen3technicalreport,openai2025gptoss120bgptoss20bmodel}.
Thus, at large scales, pretraining time would benefit greatly from an algorithm that uses fewer rectangular multiplications and more small symmetric ones.
We develop such an algorithm in \Cref{sec:gram-ns}.
\end{itemize}

\subsubsection{Kimi's Success}
Muon has been scaled successfully by Moonshot AI \citep{liu2025muon, kimiteam2026kimik2openagentic}.
Their success was enabled by an alignment of several factors:
\begin{enumerate}[nosep, leftmargin=*, labelindent=1em]
\item Earlier releases of PyTorch and Megatron-LM used DP-sharding strategies for optimizer states that were, by chance, favorable for Muon \citep{liu2025proof}.
Model and optimizer states were stored in large contiguous flat buffers, and data-parallel shards were produced by splitting this buffer.
As a result, only tensors crossing a DP boundary required an additional gather.
These advantageous strategies, however, have been deprecated in more recent releases.
\item They adopt a fine-grained MoE architecture with only a single dense layer (even fewer than in DeepSeek-V3 \citep{liu2024deepseek}), so most matrices remain small even at the one-trillion-parameter scale, keeping the Newton-Schulz overhead manageable.
\item Their main distributed training strategy combines pipeline parallelism with expert parallelism, which naturally distributes Muon's computation across devices with minimal communication. 
\end{enumerate} 

In sum, Muon has been successfully scaled, but its success relies on a subtle alignment of architectural, parallelism, and framework factors.  
For Muon to serve as a general-purpose replacement for Adam, it would benefit from cheaper, more flexible scaling properties that relax these constraints.
This is the goal of the present work.

\section{Comparison with Related Work}
As Muon has gained popularity, successive work has sought to improve it in several ways.
Most of these proposals (e.g.\ NorMuon) modify Muon's update rule so as to reach a given loss in fewer training steps; however, they use the same Newton-Schulz routine described above and generally suffer from the same scaling challenges.

\subsection{Improving Newton-Schulz}
A few papers \emph{have} attempted to improve the Newton-Schulz step itself---e.g., by optimizing the sequence of polynomials $(a_t, b_t, c_t)$ or the normalization step \citep{amsel2026the,grishina2026acceleratingnewtonschulziterationorthogonalization,boissin2025turbomuonacceleratingorthogonalitybasedoptimization}---but they retain the form of \cref{alg:standard-ns}.
In contrast, our Gram Newton-Schulz algorithm departs from this form.
Since its output is still mathematically identical to the standard version, it remains compatible with nearly all varieties of Muon, including prior improvements to Newton-Schulz.

Gram Newton-Schulz is closely akin to a method proposed in Appendix J of \citet{amsel2026the}.
Both aim to reduce the FLOP cost of Newton-Schulz, both form the Gram matrix to reduce the number of $m \times n$ matrix multiplications, and both are mathematically identical to standard Newton-Schulz.
However, our work supersedes \citet{amsel2026the} in several ways.
First, the precise formulas of Gram Newton-Schulz are different and, we believe, more stable. 
Second, we use symmetric matrix multiplication kernels; the opportunity to use these kernels more is an essential advantage of Gram Newton-Schulz not studied previously.
Third, we undertake a thorough stability analysis and provide practical recommendations that allow Gram Newton-Schulz to be used in practice with minimal ad-hoc hyperparameter tuning.
The iterative part of the method, which approximates the inverse square root of the Gram matrix $\mQ_T \approx \mR_0^{-1/2} := (\mX\mX^\top)^{-1/2}$, generalizes and speeds up the method of \citet{lakic} (see also \citet[Eq.~7.18]{doi:10.1137/1.9780898717778}), but Laki\'c does not consider the symmetric kernels, the stability issues, or the application to $\operatorname{polar}(\mX)$.

The idea of exploiting symmetry to reduce the arithmetic cost of computing $\mA^\top \mA$ has appeared before, both in standard linear algebra packages like BLAS and in the context of Muon's Newton-Schulz routine \citep{newhouse2024faster,modded_nanogpt_pr109,lin2025flash}.
However, our results show that the true potential of this trick is only realized when combined with Gram Newton-Schulz, which requires more general symmetric operations like $\mA \mB + \beta \mC$.
Furthermore, our symmetric matrix multiplication kernels are written in CuteDSL, exploiting the advanced features of NVIDIA's Hopper and Blackwell GPUs.

\subsection{Orthogonalizing a Smaller Matrix}
Our update rule is part of a second line of work that reduces the cost of the Newton-Schulz step by shrinking the size of its input, an approach pioneered by Dion \citep{ahn2025dion}.
Dion constructs a low-rank approximation of the momentum matrix and orthogonalizes only this approximation.
If $\mM \approx \mM \widehat{\mV} \cdot \widehat{\mV}^\top$ is a rank-$k$ approximation with $\widehat{\mV}^\top \widehat{\mV} = \mI \in \R^{k \times k}$, then
\begin{equation}\label{eq:low-rank-polar}
\operatorname{polar}(\mM) \approx \operatorname{polar}(\mM \widehat{\mV} \widehat{\mV}^\top) = \operatorname{polar}(\mM \widehat{\mV}) \widehat{\mV}^\top.
\end{equation}
Because the dimensions of $\mM \widehat{\mV}$ are much smaller than those of $\mM$, this approximation greatly reduces the runtime of Newton-Schulz.
In Dion, $\widehat{\mV}$ is found by a warm-started power-iteration procedure. %
If $\widehat{\mV}$ spans the top-$k$ right singular vectors of $\mM$, then this approximation gives the optimal rank-$k$ approximation of $\operatorname{polar}(\mM)$.
\citet{ahn2025dion} found that approximating $\widehat{\mV}$ from only a single warm-started power-iteration suffices for good downstream performance.
A key element of Dion's success is \emph{error feedback}, a technique that helps offset the error introduced by low-rank approximation in future iterations.
We adopt error feedback and describe it fully in \Cref{sec:nordion2}.

\textsc{Trion} \citep{modoranu2026trion} adopts the main ideas of \textsc{Dion}, but changes how the low-rank approximation is computed.
Instead of constructing $\widehat{\mV}$ via power iteration, it selects $k$ columns from the discrete cosine transform matrix.
Surprisingly, this simpler approximation performs even better than \textsc{Dion}, suggesting that finding a \emph{good} low-rank approximation is not essential; the error-feedback mechanism compensates even for large differences between $\mM$ and $\mM \widehat{\mV} \widehat{\mV}^\top$.

We push this simplification further by using the simplest possible ``low-rank approximation'' of $\mM$: we select $k$ of its rows or columns and do not operate on the others.
This achieves the same savings in the runtime of Newton-Schulz as Dion and Trion, but makes all the other steps of the algorithm (construction of $\mV$, error feedback, etc.) much easier to implement, especially in the distributed setting.

Besides low-rank approximation, block orthogonalization is an alternative way to shrink the size of the input to Newton-Schulz.
This approach partitions the momentum matrix into blocks---each typically corresponding to a shard in weight-sharded settings---and orthogonalizes each block separately.
Several papers have proposed versions of this approach \citep{boreiko2025towards,khaled2026muonbp,tang2026hierarchicalmuontilednewtonschulz}.
In particular, \citet{khaled2026muonbp} found success by periodically alternating block-wise steps with full steps, a method they call \textsc{MuonBP}.

Both low-rank and block orthogonalization methods use Newton-Schulz as a black box.
Therefore, they are trivial to combine with improvements to Newton-Schulz.

\section{Gram Newton-Schulz}
\label{sec:gram-ns}
Our first contribution is an orthogonalization algorithm that uses fewer rectangular matrix multiplications than standard Newton-Schulz.
Instead of iterating directly on the large input matrix $\mX$, we iterate on the small symmetric Gram matrix $\mX \mX^\top$.
The output of our algorithm is mathematically identical to that of standard Newton-Schulz, but it is significantly cheaper to compute.

At a high level, our strategy is based on the following formula.
If $\mX \in \mathbb{R}^{n \times m}$ with $n \leq m$, then $\mathrm{polar}(\mX) = (\mX \mX^\top)^{-1/2} \mX$.
Rather than use an iterative method to approximate $\mX_T \approx \mathrm{polar}(\mX)$ directly, we instead:
\begin{enumerate}
\item Compute the $n \times n$ Gram matrix $\mX \mX^\top$
\item Use an iterative method to approximate $\mQ_T \approx (\mX \mX^\top)^{-1/2}$
\item Output $\mQ_T \mX$
\end{enumerate}

Step 2---which comprises almost all of the algorithm's wall clock runtime and FLOP cost---works entirely with small $n \times n$ symmetric matrices.
This version uses just two rectangular matrix multiplications: $\mX \mX^\top$ in the beginning and $\mQ_T \mX$ at the end.
It also synergizes well with our symmetric GEMM kernels (\Cref{sec:kernels}).
Because we use more symmetric multiplications than before, these kernels provide an even greater speedup.
Since our algorithm works on the $n \times n$ Gram matrix of $\mX$, we call it ``Gram Newton-Schulz''.

How can we ensure that the output matches that of standard Newton-Schulz?
Recall that Newton-Schulz outputs $(p_T \circ \cdots \circ p_1)(\mX)$, where each $p_t$ is an odd polynomial $p(x) = ax + bx^3 + cx^5$.
Any odd polynomial can be rewritten in the form $p(x) = xh(x^2)$, where $h$ is a lower-degree polynomial with the same coefficients, like $h(x) = a + bx + cx^2$.
Intuitively, if $p(x) \approx 1$, then $h(y) = p(y^{1/2})y^{-1/2} \approx y^{-1/2}$, so the Newton-Schulz polynomials implicitly provide a way to approximate inverse square roots.
If we use this approximation in step 2 of Gram Newton-Schulz, then its output will match that of standard Newton-Schulz.

Formally, Gram Newton-Schulz is based on the following theorem.
In effect, it shows how to compute $\mX_T$ from $\mX_0$ without ever constructing the intermediate values $\mX_1, \ldots, \mX_{T-1}$:
\begin{theorem}\label{thm:gram-ns}
If $p_t(x) = xh_t(x^2)$ for all $t \in \{1, \ldots, T\}$, then $(p_T \circ \cdots \circ p_1)(x) = q_T x$, where $q_T$ is defined by the iteration $r_0 = x^2,\,\,q_0 = 1,$ and
\begin{align*}
z_t = h_t(r_{t-1}),\qquad\qquad
r_t = r_{t-1}z_t^2,\qquad\qquad
q_t = q_{t-1}z_t
\end{align*}
for all $t \in \{1, \ldots, T\}$.
\end{theorem}

\begin{proof}
Define $x_0 = x$ and $x_t = p_t(x_{t-1})$ for $t \in \{1, \ldots, T\}$.
We will show by induction that $r_t = x_t^2$ and $q_t = x_t / x_0$ for all $t$.
The base case $t = 0$ holds by the definition $r_0 = x^2, q_0 = 1$.
Now assume the hypothesis holds for $t-1$.
By assumption,
\[
x_t = p_t(x_{t-1}) = x_{t-1} h_t(x_{t-1}^2)
\]
By the inductive hypothesis, $h_t(x_{t-1}^2) = h_t(r_{t-1}) = z_t$, so $x_t = x_{t-1} z_t$.
Squaring both sides,
\[
x_t^2 = x_{t-1}^2 z_t^2 = r_{t-1} z_t^2 = r_t
\]
If we instead divide both sides by $x_0$ and apply the other part of the inductive hypothesis, we get
\[
\frac{x_t}{x_0} = \frac{x_{t-1}}{x_0}z_t = q_{t-1} z_t = q_t
\]
Thus, both parts of the hypothesis hold for $t$.
Finally, conclude $(p_T \circ \cdots \circ p_1)(x) = x_T = q_T x_0$.
\end{proof}

Note that, as an immediate corollary of the proof, $q_t = x_t / x_0 \to 1/x_0 = \left(x_0^2\right)^{-1/2}$.
In effect, this shows that $\mQ_T \to (\mX \mX^\top)^{-1/2}$.

To obtain our initial version of Gram Newton-Schulz, we simply lift the iteration from \Cref{thm:gram-ns} to matrices.
As in standard Newton-Schulz, each matrix operation preserves singular vectors.
Therefore, each singular value of $\mR_t$, $\mQ_t$, and $\mZ_t$ evolves independently of the others according to the scalar iteration described above.
Note that while this algorithm is mathematically equivalent to standard Newton-Schulz, it is not yet practical due to numerical instability.
The only difference between our proposed method (\Cref{alg:stable-gram-ns}) and this naive version is the presence of what we call a ``restart'' at the beginning of iteration 3 of the loop.
We will motivate this change below (\Cref{sec:stable}).

\begin{algorithm}[H]
\caption{Naive Gram Newton-Schulz}
\label{alg:naive-gram-ns}
\begin{algorithmic}[1]
\Require $\mX \in \mathbb{R}^{n \times m}$ with $n \leq m$, coefficients $\{(a_t, b_t, c_t)\}_{t=1}^5$
\State $\mX \gets \mX \,/\, (\lVert\mX\rVert_{\mathsf{F}} + \epsilon)$ \Comment{Normalize sing vals to $[0, 1]$, $\epsilon = 10^{-7}$}
\State $\mR_0 = \mX \mX^\top$
\State $\mQ_0 = \mI$
\For{$t = 1, \ldots, 5$}
    \State $\mZ_t \gets a_t\mI + b_t \mR_{t-1} + c_t \mR_{t-1}^2$ \Comment{Apply $h_t(\mR_{t-1})$}
    \State $\mQ_t \gets \mQ_{t-1} \mZ_t$
    \State $\mR_t \gets \mZ_t \mR_{t-1} \mZ_t$
\EndFor
\State \Return $\mQ_5 \mX$
\end{algorithmic}
\end{algorithm}

\subsection{Runtime of Naive Gram Newton-Schulz}
We now calculate the FLOP count of this new algorithm to show how its runtime improves on standard Newton-Schulz.
There are four matrix multiplications per iteration; if we use our symmetric GEMM kernel, these cost $n^3$ FLOPs each. 
The initialization ($\mX \mX^\top$) and output ($\mQ_5 \mX$) steps cost $mn^2$ and $2mn^2$, respectively, since $\mQ_5 \mX$ is not symmetric.
Computing $\mQ_1 = \mQ_0 \mZ_1$ is free since $\mQ_0 = \mI$, and we can skip computing $\mR_5 = \mZ_5 \mR_4 \mZ_5$; together this saves $3n^3$ FLOPs.
Thus, the total cost is $T\cdot4n^3 + 3mn^2 - 3n^3 = (4T + 3\alpha - 3)n^3$ FLOPs.

Compare this to standard Newton-Schulz's $T(3\alpha + 1)n^3$ FLOPs when using symmetric GEMMs.
When $\alpha = 1$, they are equal.
When $\alpha > 1$, Gram Newton-Schulz is cheaper, often significantly so.
For a typical Muon application ($T=5, \alpha = 4$\footnote{Transformers' MLP blocks typically have an intermediate dimension $4\times$ the model's hidden dimension.}), it saves \textbf{55\%} of the FLOPs used by standard Newton-Schulz with symmetric GEMMs, or \textbf{68\%} compared to a typical implementation without symmetric GEMMs.
For larger $T$ and $\alpha$, the savings are even greater, as the leading order term is $O((T+\alpha)n^3)$ instead of $O(T\alpha n^3)$.

In practice, when $\alpha=1$, we fall back to standard Newton-Schulz with our symmetric GEMMs (\Cref{sec:kernel-optimizations-standard-ns}), since it launches fewer GEMMs and has a faster wall clock time.

\subsection{Stabilizing Gram Newton-Schulz}
\label{sec:stable}

As written, \cref{alg:naive-gram-ns} destabilizes Muon and can even diverge.
The leading cause of this instability is the introduction of spurious negative eigenvalues in the Gram matrix $\mX \mX^\top$ due to large rounding errors in half-precision arithmetic.
These negative eigenvalues theoretically should not exist in a Gram matrix like $\mX \mX^\top$.
The main loop of Gram Newton-Schulz approximates inverse square root of positive numbers, but it diverges for negative inputs.
We provide a thorough theoretical analysis and numerical experiments showing the impact of spurious negative eigenvalues in the Gram matrix in \Cref{sec:instability-naive-gram-ns}.

In practice, we can fully mitigate this instability using a restarting strategy.
Instead of running all five iterations of the loop in a single pass and outputting $\mQ_5 \mX$, we run only the first two iterations and compute $\mX_2 = \mQ_2 \mX$.
We then ``restart'' the algorithm treating $\mX_2$ as the new input; we construct the Gram matrix $\mX_2 \mX_2^\top$, initialize $\mQ_2 = \mI$, and proceed with the remaining three iterations of the loop.
This resets any spurious negative eigenvalues to near zero and restores commutativity of $\mX$, $\mQ$, and $\mR$ at the cost of $3(\alpha - 1)n^3$ FLOPs.
We find this strategy is sufficient to preserve training quality.

To select the best iteration after which to restart, we sweep for the restart location that provides the best bound on the condition number of $\mQ_t$, assuming that forming the Gram matrix introduces spurious eigenvalues as negative as $-4\cdot10^{-4}$.
We walk through an example of how to automate this sweep in \Cref{sec:stabilizing-gram-ns}.

\Cref{alg:stable-gram-ns} presents our stabilized, training-ready version of Gram Newton-Schulz, which uses the restart strategy as well as a slight reformulation of the intermediate polynomials and a switch to \texttt{float16} instead of \texttt{bfloat16}.
The latter two changes are motivated in \Cref{sec:computing-matrix-quadratics,sec:float16-vs-bfloat16}.
Our algorithm is implemented in a pip-installable package called \href{https://github.com/Dao-AILab/gram-newton-schulz}{\texttt{gram-newton-schulz}}, which uses the Polar Express coefficients  \citep{amsel2026the}.

\begin{algorithm}[H]
\caption{Stabilized Gram Newton-Schulz}
\label{alg:stable-gram-ns}
\begin{algorithmic}[1]
\Require $\mX \in \mathbb{R}^{n \times m}$ with $n \leq m$, coefficients $\{(a_t, b_t, c_t)\}_{t=1}^5$
\State $\mX \gets \mX \,/\, (\lVert\mX\rVert_{\mathsf{F}} + \epsilon)$ \Comment{Normalize sing vals to $[0, 1]$, $\epsilon = 10^{-7}$}
\State $\mX \gets \texttt{float16}(\mX)$ \Comment{Cast to half precision for speed}
\State $\mR_{0} \gets \mX \mX^\top$
\State $\mQ_{0} \gets \mI$
\For{$t = 1, \ldots, 5$}
    \If{$t = 3$} \Comment{Restart to stabilize}
        \State $\mX \gets \mQ_2 \mX$
        \State $\mR_2 \gets \mX \mX^\top$
        \State $\mQ_2 \gets \mI$
    \EndIf
    \State $\mZ_t \gets b_t \mR_{t-1} + c_t \mR_{t-1}^2$
    \State $\mQ_t \gets \mQ_{t-1} \mZ_t +  a_t\mQ_{t-1}$ \Comment{$\mQ_t = \mQ_{t-1} h_t(\mR_{t-1})$}
    \State $(\mathbf{RZ})_t \gets \mR_{t-1} \mZ_t + a_t\mR_{t-1}$
    \State $\mR_t \gets \mZ_t (\mathbf{RZ})_t + a_t(\mathbf{RZ})_t$ \Comment{$\mR_t = \mR_{t-1} h_t(\mR_{t-1})^2$}
\EndFor
\State $\mX \gets \mQ_5 \mX$
\State \Return $\mX$
\end{algorithmic}
\end{algorithm}

\section{Symmetric GEMM Kernels in CuteDSL}
\label{sec:kernels}
Our second contribution is a set of custom GPU kernels for the operations $\mA \mB$ and $\alpha \mA \mB + \beta \mC$ that assume $\mA \mB$ and $\mC$ are symmetric.
These \emph{symmetric GEMM} kernels compute the lower triangle of the output matrix in the usual way and copy the results to the upper triangle (\Cref{fig:gemm-benchmarks}, left), saving about half the floating point operations used by standard matrix multiplication routines.
Our kernels allow us to take advantage of the symmetric structure of Gram Newton-Schulz.
As noted above, they also accelerate standard Newton-Schulz, but their impact is far greater when combined with Gram Newton-Schulz.
We target the Hopper and Blackwell GPU architectures.
\Cref{fig:gemm-benchmarks} (right) shows that our kernels achieve superb performance on both architectures, significantly outperforming the standard GEMM routine from cuBLAS across a range of matrix sizes.

\begin{figure}[h]
\centering
\includegraphics[width=0.435\textwidth]{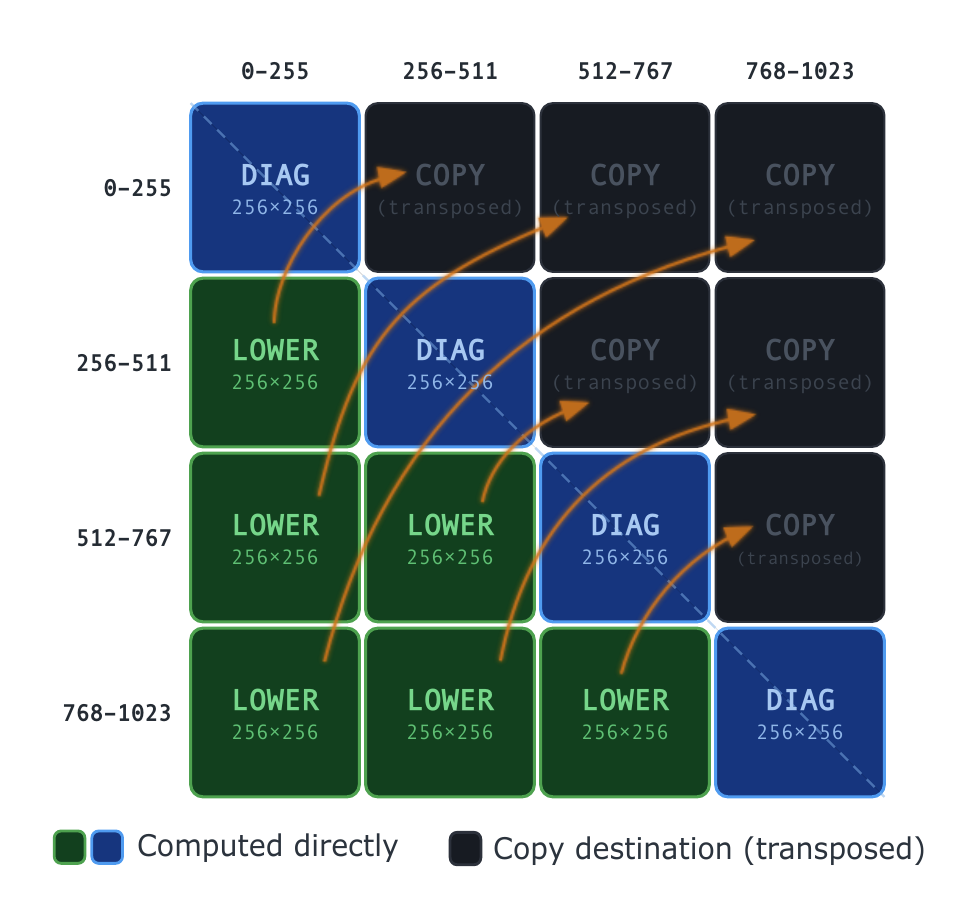}
\hfill
\includegraphics[width=0.53\textwidth]{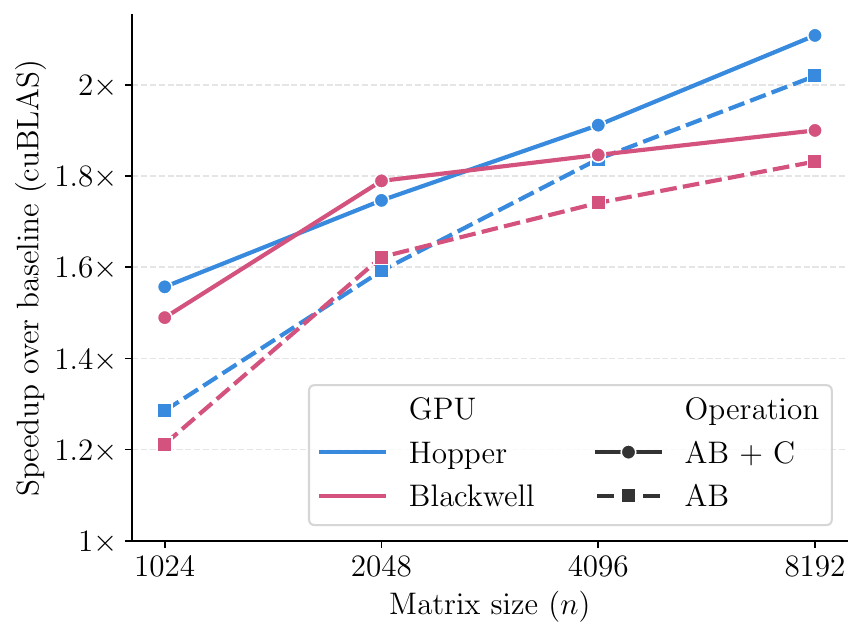}
\caption{\textbf{Left:} Symmetric GEMM computes $256 \times 256$ tiles from the lower triangle and main diagonal, then transposes and copies each lower tile to the corresponding upper tile.
\textbf{Right:} Our CuteDSL symmetric GEMM kernels benchmarked against cuBLAS GEMM kernels on Hopper and Blackwell GPUs.
Input matrices $\mA, \mB, \mC$ have dimensions $n \times n$.
For large enough $n$, our kernels achieve a $\sim 2\times$ speedup over cuBLAS, both with and without an epilogue addition of $\mC$.}
\label{fig:gemm-benchmarks}
\end{figure}

Most GEMM kernels have the following form:
\begin{enumerate}
\item Scheduler: The output matrix is divided into tiles.
A schedule is created that assigns each tile to a group of workers.
\item Computing each output tile of $\mA \mB$ or $\alpha \mA \mB + \beta \mC$ requires three steps:
\begin{enumerate}
    \item Prologue: The rows of $\mA$ and columns of $\mB$ needed for the current tile are loaded in from general memory (high-bandwidth memory) to shared memory (SRAM).
    \item Matrix-Multiply Accumulate (MMA): The rows and columns are multiplied and written to the register file (Hopper) or tensor memory (Blackwell).
    \item Epilogue: Additional tensors needed for fused operations (e.g.\ $\mC$, $\alpha$, $\beta$) are loaded, fused operations are executed, and the final output is written from the register file to shared memory and then to general memory.
\end{enumerate}
\end{enumerate}

Our symmetric GEMM kernels differ from the standard one only in their schedulers and epilogues.

\paragraph{Triangular Scheduler}
In the standard GEMM, the entire output matrix is partitioned into work tiles that are load balanced and evenly partitioned amongst clusters of thread blocks, where thread blocks in the same cluster can access the same shared memory and are therefore scheduled to run together.
Each cluster then computes its assigned work tiles in succession.
Our tile scheduler in the symmetric GEMM is almost identical, except that it only partitions work tiles from the lower triangle of the matrix (including the main diagonal); work tiles in the upper triangle are unassigned.
This \emph{triangular scheduler} ensures that clusters are load balanced and no unnecessary work is performed.

\paragraph{Epilogue: Writing to the Transposed Tile}
In the GEMM epilogue, when the computed values of the lower triangle (excluding the main diagonal) are written to their assigned tile in general memory (HBM), they are also written to their transposed tile location in the upper triangle.
The left panel of \Cref{fig:gemm-benchmarks} illustrates this process.

We describe the details of our symmetric kernels and of our kernelized implementation of standard Newton-Schulz in \Cref{app:kernel-details}.

\section{The Dion3 Update Rule}
\label{sec:nordion2}
Our third contribution is a variant of Muon that avoids orthogonalizing the full momentum matrix, making the optimizer even faster.
At each step, we select only a fraction of the rows from the momentum matrix and perform the Muon update (including the orthogonalization) as if the other rows did not exist.
In words, the main steps of the Dion3 update rule are as follows:
\begin{enumerate}
    \item \textbf{Select} a fraction of the rows (or columns) of the momentum matrix.
    \item \textbf{Orthogonalize} the selected submatrix using Gram Newton-Schulz.
    \item \textbf{Update} the selected rows of the weight matrix using the output of the previous step.
    Do not update the other rows.
    \item \textbf{Error Feedback} Decay \emph{only} the selected rows of the momentum matrix by a multiplicative factor.
\end{enumerate}
Unlike our other contributions, this procedure changes the optimization trajectory, but experiments in \Cref{sec:exp-nordion2} show that this change is benign.
When \emph{all} rows are selected (i.e., the fraction equals 1), our update rule reduces to Muon in exact arithmetic.
Note that an earlier version of this update rule appeared in an unpublished report under the name Dion2.
As it was not formally published, we drop the old name and present it here as part of Dion3.
We now describe each step of the optimizer in detail.
\Cref{alg:nordion2} gives the full pseudocode.

\begin{algorithm}[H]
\caption{Dion3 update rule (single weight matrix)}
\label{alg:nordion2}
\begin{algorithmic}[1]
\Require Weight $\mW \in \mathbb{R}^{n \times m}$, gradient $\mG$, momentum buffer $\mM$, row-wise 2\textsuperscript{nd} moment buffer $\vv \in \mathbb{R}^{n}$, fraction $f \in (0, 1]$, momentum $\mu$, decay $\beta_2$, learning rate $\eta$.
\State $\mM \gets \mM + \mG$ %
\State $\mathcal{S} \gets$ indices of the $k = \lceil f n \rceil$ rows of $\mM$ with largest $\ell_1$ norm \Comment{Selection}
\State $\mO \gets \operatorname{polar}(\mM[\mathcal{S},:])$ \Comment{Orthogonalize $k \times m$ submatrix}
\State $\vv_i \gets \beta_2 \vv_i + (1-\beta_2) \tfrac{1}{m}\sum_j \mO_{ij}^2$ \quad for $i \in \mathcal{S}$ \Comment{Optional: NorMuon steps...}
\State $\widehat{\mO}_{ij} \gets \mO_{ij} / (\sqrt{\vv_i} + \epsilon)$
\State $\widehat{\mO} \gets \widehat{\mO}\,\lVert\mO\rVert_{\mathsf F}/\lVert\widehat{\mO}\rVert_{\mathsf F}$
\State $\mW \gets (1 - \eta \cdot \texttt{wd})\mW$ \Comment{Optional: weight decay}\label{step:weight-decay}
\State $\mW_{ij} \gets \mW_{ij} - \eta\,\widehat{\mO}_{ij}$\quad for $i \in \mathcal S$ \Comment{Update selected rows}\label{step:weight-update}
\State $\mM_{ij} \gets \mu\,\mM_{ij}$ \quad for $i \in \mathcal S$\Comment{Error Feedback}
\end{algorithmic}
\end{algorithm}

\paragraph{Selection}
Our update rule introduces a new hyperparameter $f \in (0, 1]$ that controls the fraction of rows (or columns) to select.
We think of $f$ as a compression factor; $f = 1$ corresponds to Muon, and decreasing $f$ speeds up the algorithm.
We recommend $f = \nicefrac14$ or $f = \nicefrac18$.
We could subselect either the rows or the columns.
When the weight matrix is sharded across one of these dimensions, we subselect along that same dimension.
Otherwise, we pick whichever dimension is smaller.
For brevity, we refer only to row selection throughout this section.

We use a simple selection strategy: we pick the $k = \lceil f n \rceil$ rows of largest $\ell_1$ norm.
Initial experiments showed that optimization quality when using random selection was not much worse, so other selection strategies may perform well.
In the distributed setting, selecting the top rows across all shards would require an extra synchronization round and a ragged all-to-all, so we simply select the top-$f$ fraction of rows from each shard.
True global selection is available as an option in our package.

\paragraph{Orthogonalization}
The benefit of our selection strategy is realized in this step.
For an input of size $n \times m$, each matrix multiplication in Newton-Schulz costs either $mn^2$ or $n^3$ FLOPs with our symmetric kernels.
Selection reduces the dimensions to $fn \times m$, slashing the cost of these multiplications by a factor of at least $1/f^2$.
This benefit is compounded by using Gram Newton-Schulz, which is faster than standard Newton-Schulz precisely when the aspect ratio $\alpha = m / n$ is large.
Happily, selection increases $\alpha$ by a factor of $1/f$.

Our optimizer also benefits the \emph{communication} cost of Muon in the distributed setting.
Rather than assemble the entire momentum matrix on a single device to perform the orthogonalization, we only assemble the selected rows, reducing the communication volume by a factor of $1/f$. 
We describe our communication strategy in detail in \Cref{sec:megabatch}.

\paragraph{Weight Update}
We experiment with two flavors of the update rule: Muon and NorMuon (corresponding to \eqref{eq:muon-update} and \eqref{eq:normuon-update}, respectively).
For the Muon flavor, we use each row of the orthogonalized submatrix $\mO = \operatorname{polar}(\mM[\mathcal S, :])$ to update the matching row of the weight matrix: $\mW[\mathcal S, :] \gets \mW[\mathcal S, :] - \eta\mO$.
The NorMuon flavor is similar, but with extra steps that update the row-wise second moments and use them to rescale $\mO$; see \Cref{alg:nordion2}.
In both flavors, if weight decay is enabled, it is applied to all rows of $\mW$ before updating the selected rows.

Our initial implementation of Dion3 suffered from a nefarious numerical issue that caused it to underperform Muon and NorMuon in final loss.
Normally, the compiler fuses the weight update into a single operation $\mW \gets (1 - \eta\cdot\texttt{wd})\mW - \eta\mO$.
Tensors are cast from \texttt{bfloat16} to \texttt{float32}, the arithmetic is performed, and the result is cast back to \texttt{bfloat16}.
When we introduced the row selection operator, it broke this automatic fusion, creating multiple rounds of upcasting and downcasting.
To restore the original numerical behavior, we implemented a custom Triton kernel for \Cref{step:weight-decay,step:weight-update} of \Cref{alg:nordion2}.
Furthermore, we found that NorMuon's normalization step should be performed in \texttt{float32} to avoid similar issues.

\paragraph{Error Feedback}
We adapt the \emph{error feedback} mechanism from Dion \citep{ahn2025dion} to our simpler selection-based strategy for approximating $\mM$.
The usual rule for updating the momentum matrix is $\mM \gets \mM + \mG;\, \mM \gets \mu \mM$, which ensures that $\mM$ is an exponential moving average of the gradients across iterations.
The main idea of error feedback is to decay only the component of $\mM$ that was captured by our approximation.
Decompose $\mM$ into two parts, the selected part $\widehat{\mM}$---which matches $\mM$ at the selected rows and contains zeros in all other rows---and the unselected part $\mM - \widehat{\mM}$.
(In Dion, $\widehat{\mM}$ would be a low-rank approximation of $\mM$.)
Instead of the usual rule $\mM \gets \mu \mM$, error feedback does $\mM \gets \mu \widehat{\mM} + (\mM - \widehat{\mM})$.
That is, it decays the selected rows by a factor of $\mu$ but leaves the other rows alone.

In effect, error feedback adds an extra $(1 - \mu)(\mM - \widehat{\mM})$ to the momentum.
By boosting the residual component, this extra term nudges future iterations to select rows that were ignored by the current iteration.
Even if a given row of $\mG$ is consistently small, error feedback allows it to build up in $\mM$ over many iterations, so it eventually gets selected and participates in the update.
When the approximation is exact ($\mM = \widehat{\mM}$), we recover standard momentum.

\paragraph{Finite Precision}
At $f = 1$ every row is selected each step, and Dion3 reduces to Muon/NorMuon up to four implementation details: the selection step permutes the rows, momentum damping is applied after forming the update\footnote{That is, Newton-Schulz sees $\mM + \mG$ rather than $\mu\mM + \mG$. 
Since $\mM$ is initialized to zero, these matrices differ by a constant factor of $\mu$, which is removed by orthogonalization anyway.}, the NorMuon normalization steps run in \texttt{float32}, and the update step uses our custom Triton kernel.
We confirmed that, despite these differences, the convergence curve of our optimizer matches that of Muon / NorMuon almost exactly (see \Cref{app:nordion2-ablations}).

\paragraph{Kernel launch minimization}
Dion3's row selection process necessarily requires more kernel launches than standard Muon does.
At small scale, the cost of these extra kernel launches can dominate.
Since the kernel dependency structure is the same from round to round, we use a standard CUDA graph capture and replay to avoid this overhead and realize the full potential of our optimizer on smaller-scale models.

\section{Megabatching and Communication}
\label{sec:megabatch}
Our final contribution concerns the distributed setting.
Under fully-sharded data parallelism (FSDP), each momentum matrix is partitioned across GPUs, so it must first be assembled onto a single device before Newton-Schulz can run.
This requires an all-to-all communication along the sharded dimension and a second all-to-all to scatter back the result.
As discussed in \Cref{sec:scalability}, the cost of these operations is a major obstacle to scaling Muon \citep{essential2025muon, ahn2025dion}.
To help users easily adopt Dion3 in a wide range of scenarios, we implement it in \texttt{dion}, a pip-installable package that gracefully manages this communication for FSDP2, DDP, and mixed sharding strategies with minimal overhead.
The package supports our kernelized implementation of Gram Newton-Schulz and both flavors of the Dion3 update rule as well as Muon, NorMuon, and standard Newton-Schulz, all of which share the same distributed backend.
We believe \texttt{dion} provides a full-stack, production-ready solution for distributed orthogonal optimization.

A key feature of our implementation is \textbf{megabatching}.
While the update rule of \Cref{sec:nordion2} reduces the communication \emph{volume} (the number of bits that must be assembled and scattered), megabatching reduces the number of \emph{rounds} of communication.
A naive implementation of Muon performs the update in batches of \texttt{world\_size} matrices at a time; each matrix in the batch is assembled and orthogonalized in parallel on a different GPU without any GPU sitting idle.
For a model with $N$ orthogonalized matrices, this requires $O(N / \texttt{world\_size})$ rounds of communication (all-to-alls) per optimizer step.
Unfortunately, launching and synchronizing each round incurs overhead.
Since this overhead is independent of the communication volume, our update rule does not mitigate it.
Furthermore, each peer-to-peer message is small---just one shard of one parameter.
These messages may fail to saturate the interconnect, so bandwidth utilization is poor.
We demonstrate these effects in \Cref{app:a2a-microbench}.

Our solution is megabatching; we group all matrices of the same shape into a \emph{single} batch.
For a given shape, the local momentum shards are packed into one all-to-all, assembled, orthogonalized as a batch, and scattered back together.
Transformers contain only a handful of distinct weight shapes, so megabatching reduces the number of communication rounds to $O(1)$, independent of model depth.
We implement this megabatching strategy in the \texttt{dion} package, where it powers Muon, NorMuon, and Dion3 alike.
Note that megabatching only changes \emph{when} data moves; the assembly, momentum-state layout, checkpoint format, and orthogonalization math are all the same.

\paragraph{Benchmarking}
We compare megabatching against the previous communication strategy by measuring the wall-clock time of the optimizer step.
We hold the model, data, and optimizer (Muon) fixed, and time both strategies back-to-back on the same GPUs.
We test two model sizes (1B and 14B parameters) and two FSDP configurations (8 or 32 GPUs).
Results are shown in \Cref{tab:megabatch}.
For the 14B model, the computational cost of Newton-Schulz dominates, so reducing communication overhead has little impact.
For the 1B model with 32 shards, each rank holds only one or two matrices of each shape anyway (see \Cref{tab:app-arch}), so the batch size is about the same for both strategies.
However, for the 1B model with 8 shards, the optimizer is communication-bound and each rank holds many matrices.
Here, megabatching has a major impact, reducing step time by $35\%$.

\begin{table}[h]
\centering
\caption{Effect of megabatched communication on the per-GPU optimizer step time (Muon, median over steps). ``Batching'' orthogonalizes matrices in groups whose size equals the shard count; ``megabatch'' uses a single batch per shape group. When the weights are small (so Newton-Schulz is cheap) and each rank holds several matrices, megabatching substantially reduces optimizer time. Each node has four GH200s.}
\label{tab:megabatch}
\begin{tabular}{@{}rrrrrr@{}}
\toprule
Model size & Nodes & FSDP Shards & Batching (ms) &  Megabatching (ms) & Change \\
\midrule
1B  & 1 & 8 & $80.7$ & $52.1$ & $-35\%$ \\
1B  & 4 & 32 & $61.9$  & $59.3$ & $-4\%$ \\
14B & 1 & 8 & $144.0$ & $140.8$& $-2\%$ \\
14B & 4 & 32 & $94.7$ & $89.1$ & $-6\%$ \\
\bottomrule
\end{tabular}
\end{table}

\paragraph{Compressed Data Parallelism}
There is a prospect for further savings at larger scales, where data parallelism spans multiple pods in a data-center network (DCN) or hybrid sharding is used \citep{zhao2023pytorch}.
A common deployment combines model parallelism within an inter-chip interconnect (ICI) with pure data parallelism across pods \citep{scaling-book}.
Because DCN bandwidth is typically far smaller than ICI bandwidth, shrinking the volume of data-parallel communication is especially valuable in this regime.
Here, Dion3 allows for compressed data-parallel synchronizations analogous to those of \textsc{Dion} \citep[\S 3.3]{ahn2025dion}.
Ordinarily, gradients are averaged across data-parallel replicas at every step so that the momentum buffers, and hence the weights, stay consistent.
With Dion3, however, we need not synchronize the entire momentum matrix at every step, only the selected submatrix $\mM[\mathcal S, :]$.
When the rows $\mathcal S$ can be selected without first synchronizing $\mM$ (e.g., by picking $\mathcal S$ at random), this strategy performs an identical update at a fraction of the communication cost.

\section{Experiments}
\label{sec:experiments}

\subsection{Model Quality Is Preserved}
\label{sec:exp-nordion2}
Of our four contributions, only our update rule explicitly changes the optimizer trajectory.
However, Gram Newton-Schulz and our symmetric GEMM kernels introduce numerical differences due to finite precision arithmetic.
Experiments in \Cref{sec:gram-ns-quality} show that these differences do not affect training.
We also confirmed that, as expected, Dion3 with selection fraction $f = 1$ matches Muon/NorMuon almost exactly (see \Cref{fig:nordion2-loss-curve-1b}).

In contrast, Dion3 with $f < 1$ is a genuinely new optimizer.
In this section, we demonstrate that Dion3 achieves a slightly \emph{better} loss than the baseline.
We did not set out to improve training quality, so this result is a small but pleasant surprise.
At a minimum, it shows that our subselection strategy does no harm to the loss.

We train 1B-parameter dense transformers on 100B tokens of the ClimbMix dataset \citep{diao2026nemotronclimb}, using fully-sharded data parallelism (FSDP) with MXFP8 weights.
(Below, we also train at larger scales.)
As always, Newton-Schulz runs in half precision arithmetic.
The models use grouped-query and sliding-window attention, but are \emph{not} mixtures of experts.
We report the cross-entropy loss (CE) of next-token prediction on a held-out validation split of ClimbMix.
For further details about our setup, see \Cref{app:setup}.
For a fair comparison, we begin by tuning the baseline.
In this section, we compare NorMuon to the NorMuon flavor of Dion3, as initial experiments gave it a slight edge over Muon.
We find that the optimal learning rate and momentum coefficient for NorMuon are $\eta = 0.01$ and $\mu = 0.9$, respectively.

\paragraph{Learning-rate transfer rule} Since Dion3 updates only a fraction of the rows in each step, it reduces the effective step size.
This in turn changes the optimal learning rate.
To understand this effect, we sweep both the fraction $f \in \{\nicefrac{1}{2}, \nicefrac{1}{4}, \nicefrac{1}{8}, \nicefrac1{16}\}$ and the learning rate $\eta$ of Dion3, along with our earlier sweep for NorMuon ($f=1$).
The results in \Cref{fig:lr-fraction-heatmap} show a clear trend: the optimal learning rate for a given $f$ scales as $\eta \propto \sqrt{1/f}$.
We can explain this transfer rule as follows.
For Muon, the size of the update is $\|\eta \cdot \operatorname{polar}(\mM)\|_{\mathsf F} = \eta \sqrt{n}$, where $n$ is the smaller dimension.
For Dion3, it is $\|\eta'\cdot\operatorname{polar}(\mM[\mathcal S, :])\|_{\mathsf F} = \eta' \sqrt{f n}$.
For these to match, we must set $\eta' = \eta / \sqrt{f}$.
We do not bother retuning the momentum coefficient; we simply copy the optimal setting from NorMuon ($\mu = 0.9$) to Dion3.

\begin{figure}[ht]
\centering
\includegraphics[width=0.54\textwidth]{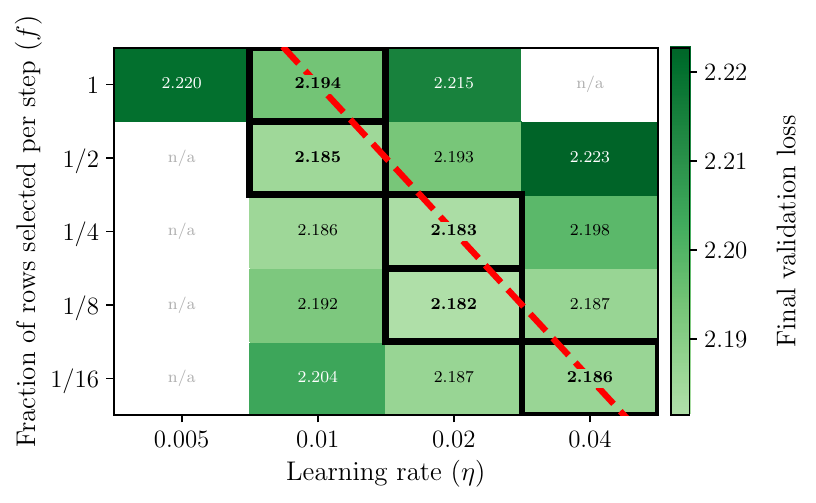}
\hfill
\includegraphics[width=0.44\textwidth]{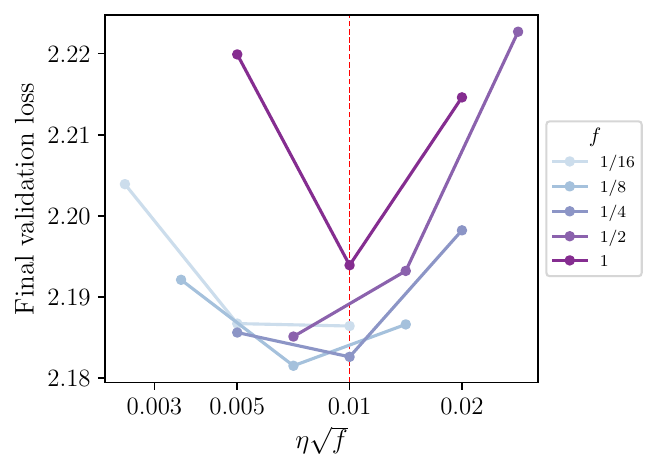}
\caption{\textbf{Left:} Final validation loss for 1B-parameter models trained on 100B tokens of ClimbMix with NorMuon or Dion3, as a function of row-selection fraction $f$ and learning rate $\eta$.
NorMuon is $f=1$.
Bolded cell in each row shows best $\eta$.
Optimal $\eta$s track the line $\eta\sqrt{f} = 0.01$, shown in red on a log-log scale.
Not shown: $f=1, \eta=0.03$ (loss = $2.237$) and $f=\nicefrac1{32}, \eta=0.04$ (loss $=2.191$).
\textbf{Right:} The same data reduced to two dimensions: validation loss as a function of $\eta \sqrt{f}$.
Final loss is minimized when $\eta \sqrt{f} \approx 0.01$.}
\label{fig:lr-fraction-heatmap}
\end{figure}

\paragraph{Dion3 improves the loss} \Cref{fig:lr-fraction-heatmap} contains another remarkable finding: Dion3 with $f<1$ actually outperforms NorMuon when tuned correctly.
Indeed, the lowest loss is achieved at $f = \nicefrac18$.
\Cref{fig:best-nordion2} shows the validation loss curve for each optimizer (excluding $f = \nicefrac14$) at its best learning rate.
It shows a clear gap throughout training between NorMuon and the Dion3 variants, which confirms this finding.

\begin{figure}[ht]
\centering
\includegraphics[width=0.9\textwidth]{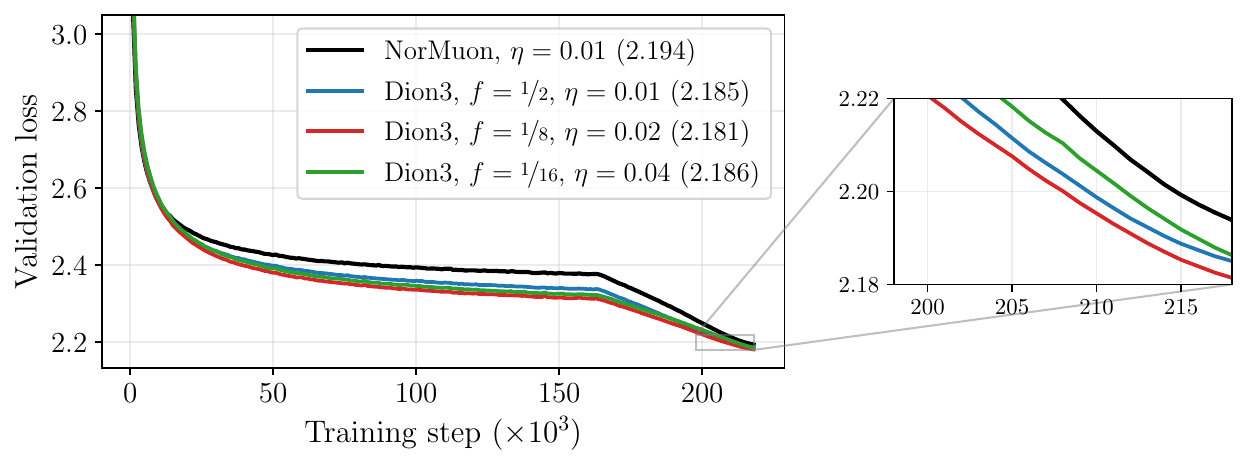}
\caption{Validation loss curves for 1B-parameter models trained on 100B tokens of ClimbMix with NorMuon or Dion3 at $f = \nicefrac12, \nicefrac18$, or $\nicefrac1{16}$.
Each optimizer uses its best learning rate from \Cref{fig:lr-fraction-heatmap}.
Parentheses show final loss.
All Dion3 variants track below the fully-tuned NorMuon baseline throughout training, finishing about $0.01$ points lower.}
\label{fig:best-nordion2}
\end{figure}

To strengthen this finding, we scale up the model size to between 3B and 14B parameters.
For reasons of cost, we now train on 10B tokens instead of 100B.
We compare Dion3 with $f = \nicefrac14$ against NorMuon, each using their best learning rate as identified above.
For each scale and optimizer, \Cref{tab:nordion2-combined} reports both validation loss on ClimbMix and downstream accuracy on a suite of 12 standard benchmarks.
In validation loss, Dion3 outperforms NorMuon at every scale, achieving the largest improvement ($-0.027$) at the largest scale (14B).
\Cref{fig:nordion2-loss-curve} shows that, as for the 1B parameter model, validation loss is lower throughout training when using Dion3.
In downstream accuracy, Dion3 wins at three of four scales, with an improvement of 0.7 percentage points at 14B.

\begin{table}[ht]
\centering
\caption{Dion3 ($f\!=\!\nicefrac{1}{4}$, $\eta\!=\!0.02$) versus NorMuon ($\eta\!=\!0.01$) across model sizes trained on 10B tokens of ClimbMix: final validation loss (cross-entropy; lower is better) and downstream accuracy (macro-average over 12 standard benchmarks: ARC easy/challenge, BoolQ, COPA, HellaSwag, LAMBADA, MMLU, OpenBookQA, PIQA, RTE, TruthfulQA, WinoGrande; higher is better).
The winner of each comparison is bolded.}
\label{tab:nordion2-combined}
\begin{tabular}{@{}rcccccc@{}}
\toprule
Model Size & \multicolumn{3}{c}{Validation Loss} & \multicolumn{3}{c}{Downstream Accuracy (\%)} \\
\cmidrule(lr){2-4} \cmidrule(lr){5-7}
 & NorMuon & Dion3 ($f\!=\!\nicefrac{1}{4}$) & $\Delta$ & NorMuon & Dion3 ($f\!=\!\nicefrac{1}{4}$) & $\Delta$ \\
\midrule
3B  & 2.269 & \textbf{2.257} & $-0.012$ & 53.9 & \textbf{54.9} & $+1.0$ \\
4B  & 2.243 & \textbf{2.232} & $-0.011$ & \textbf{55.2} & 54.9 & $-0.3$ \\
7B  & 2.220 & \textbf{2.206} & $-0.014$ & 56.0 & \textbf{56.1} & $+0.1$ \\
14B & 2.189 & \textbf{2.162} & $-0.027$ & 57.4 & \textbf{58.1} & $+0.7$ \\
\bottomrule
\end{tabular}
\end{table}

\begin{figure}[ht]
\centering
\includegraphics[width=0.9\textwidth]{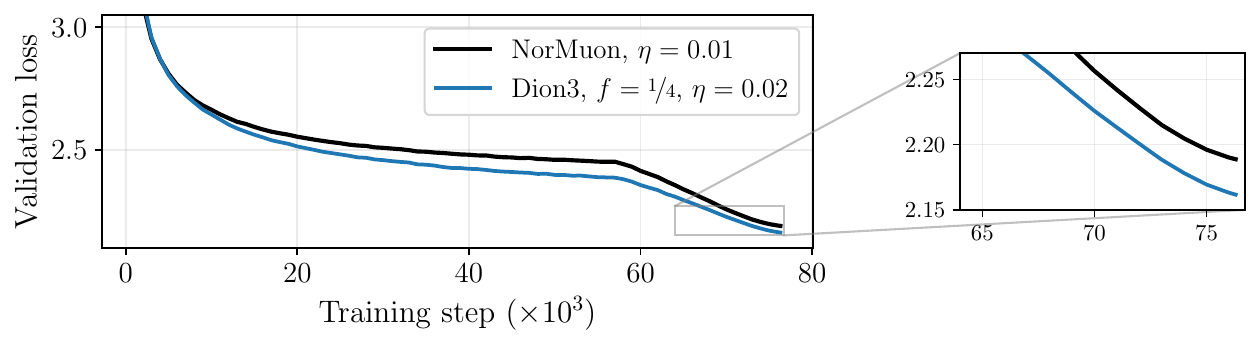}
\caption{Validation loss over training at 14B for Dion3 ($f\!=\!\nicefrac{1}{4}$, $\eta\!=\!0.02$) versus NorMuon ($\eta\!=\!0.01$), trained on 10B tokens of ClimbMix.
Dion3 tracks below NorMuon throughout.}
\label{fig:nordion2-loss-curve}
\end{figure}

\subsection{Dion3 Accelerates the Optimizer}
\label{sec:opt-speedups}
We now measure how much Dion3 reduces the cost of Muon's optimizer step.
Starting with standard Muon, we successively add our improvements---symmetric kernels, Gram Newton-Schulz, and fractional updates with $f = 0.5$ or $f = 0.25$---stacking each on top of the previous ones.
(All versions use megabatching.)
We also compare against plain AdamW.
We run each version of the optimizer on models of various sizes either on 1 GPU or using FSDP on 4 GPUs.
We use CUDA events to measure the GPU time of the optimizer step.
For fidelity, we run full training steps with forward and backward passes (on synthetic data) but these are excluded from the recorded timings.

\begin{figure}[ht]
\centering
\includegraphics[width=\textwidth]{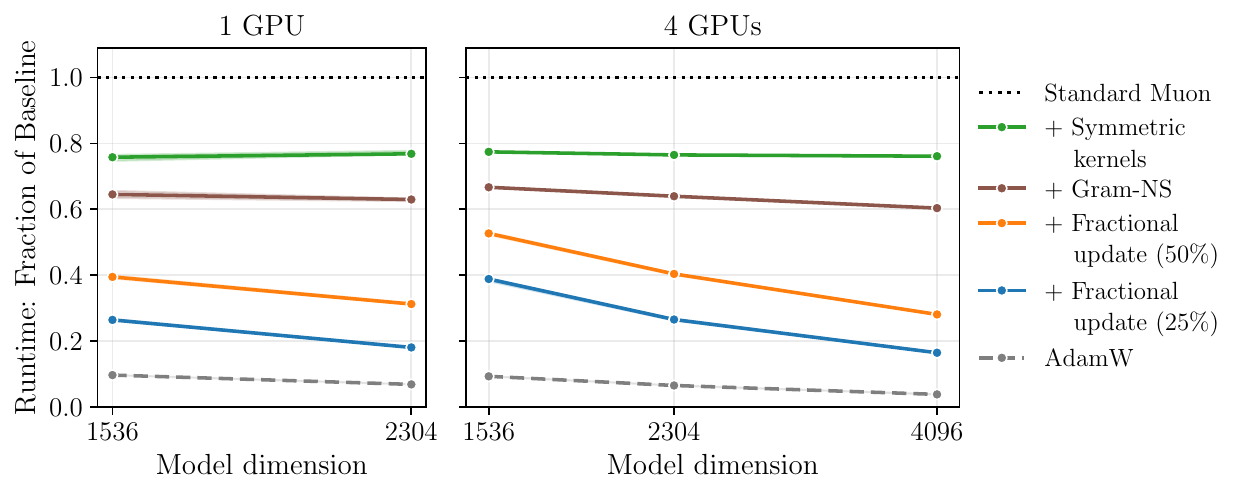}
\caption{Optimizer step time (excluding forward/backward pass) relative to standard Muon across model scales when training on 1 GH200 (left) and FSDP over 4 GH200s (right).
Each colored line adds one of our contributions on top of the previous one; AdamW is shown for reference.
Lines show median over 25 steps and bands show interquartile range.
Adding each of our contributions consistently reduces runtime; together they achieve a $6\times$ speedup for larger models.}
\label{fig:optstep-relative-muon}
\end{figure}

Results are shown in \Cref{fig:optstep-relative-muon}.
A subset of these results appears in \Cref{fig:optstep-bar}, and the corresponding timings for the NorMuon family are given in \Cref{app:nordion2-ablations} (\Cref{fig:normuon-optstep-relative-muon}).
Each part of Dion3 yields a consistent speedup across scales and parallelism configurations.
Our symmetric kernels and Gram Newton-Schulz give a combined speedup of $1.5\times$ or more over standard Muon.
The fractional update rule cuts the runtime dramatically.
As expected, its impact grows with model size, as the cubic computational cost of Newton-Schulz increasingly dominates the other operations.
For sufficiently large models, setting $f = \nicefrac12$ and $f = \nicefrac14$ gives further $2\times$ and $3.7\times$ reductions respectively, for overall speedups of $3.6\times$ and $6.5\times$.
Thus, while Dion3 is still slower than AdamW, the gap is \emph{significantly} smaller.

In \Cref{app:timing-metrics}, we report CPU time and communication volume in addition to GPU time, showing that CPU overhead is small and that fractional updates reduce communication by a factor of $1/f$.
The communication savings of our fractional update rule can be crucial in some settings, though the setting of \Cref{fig:optstep-relative-muon} is compute-bound: neither communication nor host time is exposed.

In \Cref{app:single-gpu}, we conduct an additional suite of benchmarks on a wider range of architectures, including open source models like Gemma and mixtures-of-experts.
On architectures like these, whose weight matrices have higher aspect ratios ($\alpha = 8$ instead of $4$), Gram Newton-Schulz and the symmetric kernels alone achieve a speedup of $2\times$ (see \Cref{fig:optim-step-hopper,fig:kimi-k2-ns}).

\section{Conclusion}
For LLMs trained with Muon, estimates suggest that optimizer step time accounts for between 1\% (e.g., \citet{abadji2026lagunam1xs2technicalreport}) and 17\% of total training time (see \Cref{app:end-to-end_time}).
While Muon has been scaled successfully, doing so has demanded particular choices of architecture and parallelism strategy, along with \emph{significant} engineering effort.
As interest in Muon continues to spread, this paradigm is too brittle; a flexible, general-purpose remedy for Muon's overhead is needed.

Dion3 allows practitioners to easily realize the benefits of Muon without paying the high cost of its orthogonalization step---even for large-scale models in highly distributed settings.
Our Gram Newton-Schulz algorithm and CuteDSL kernels speed up Muon by $1.5\times$ for dense models and $2\times$ for MoEs, a rare case of free lunch performance.
Megabatching has an equally large effect in certain distributed settings.
Dion3's fractional update rule ($f = \nicefrac14$) provides an \emph{additional} $3.7\times$ speedup.
Though this update rule changes the optimizer trajectory, we find that it actually improves training quality in our setting.
This improvement is unexpected, as we designed Dion3 to cheaply approximate Muon rather than improve upon it.
Further work is needed to determine how widely this improvement generalizes, but we find support in a recent result similar to our own: \citet{joo2026surprisingeffectivenessmaskingupdates} showed that randomly masking blocks of the update improves the trajectory of SGD with momentum.
Overall, Dion3 as implemented in our \texttt{dion} and \texttt{gram-newton-schulz} packages provides the tools to make orthogonal optimizers practical and accessible in a wide range of settings.

\paragraph{Acknowledgements}
The authors thank Zichong Li for contributing an implementation of NorMuon to the \texttt{dion} repo.
NA is supported by NSF award 2234660.

\printbibliography[heading=bibintoc]

\clearpage
\begin{appendices}
\crefalias{section}{appendix}
\crefalias{subsection}{appendix}
\crefalias{subsubsection}{appendix}

\section{Alternative Experiments on Gram Newton-Schulz}
\label{app:single-gpu}
This section presents additional experimental results about Gram Newton-Schulz and the symmetric GEMM kernels.
Complementing \Cref{sec:experiments}, these experiments study a wider range of architectures but at smaller scale, using only one GPU at a time.

\subsection{Setup}
\label{app:onegpu-setup}
We train four architectures: Llama-430M, Qwen-600M, Gemma-1B, and a custom MoE architecture with 1B parameters, of which ${\sim}20\%$ are active \citep{grattafiori2024llama3herdmodels,yang2025qwen3technicalreport,gemmateam2025gemma3technicalreport}.
We train on FineWeb-Edu \citep{lozhkov2024fineweb-edu}.
The number of training tokens for each dense model is given by the Chinchilla scaling law; for MoE-1B it is twice the Chinchilla scaling law with respect to the active parameters.
We use a cosine learning rate scheduler with the base learning rates given in \Cref{tab:blog-base-lr}.

We use Muon on most matrix parameters---the attention layer's projection matrices ($\mW_Q, \mW_K, \mW_V$), the projection following attention ($\mW_O$), the SwiGLU MLP weights ($\mW_{\mathrm{MLP,UP}}$, $\mW_{\mathrm{MLP,GATE}}$, $\mW_{\mathrm{MLP,DOWN}}$), and the token router of the MoE ($\mW_{\mathrm{router}}$)---but exclude the embedding and unembedding layers.
Although we are not in the distributed setting, we use megabatching (\Cref{sec:megabatch}), as batching Newton-Schulz's GEMM operations makes them faster.
As is standard for Muon, we adjust the learning rate by scaling the update for each weight matrix based on its dimensions.
For these experiments, we found that Moonshot AI's strategy of scaling by $0.2 \sqrt{\max(\mathrm{fan\_out}, \mathrm{fan\_in})}$ yields the best loss curves \citep[\S 2.2]{liu2025muon}.

\begin{table}[h]
\centering
\caption{Architecture, base learning rate, and training-token budget of each model.}
\label{tab:blog-base-lr}
\begin{tabular}{@{}lrrrrr@{}}
\toprule
Model & Hidden dim & Layers & MLP dim & Learning rate & Training tokens \\
\midrule
Llama-430M & $1024$ & $24$ & $4096$   & $3 \times 10^{-3}$   & $10$B \\
Qwen-600M  & $1024$ & $28$ & $3072$   & $1.5 \times 10^{-3}$ & $12$B \\
Gemma-1B   & $2048$ & $8$  & $16384$  & $3 \times 10^{-4}$   & $22$B \\
MoE-1B     & $1024$ & $9$  & (per-expert) $256$    & $2.5 \times 10^{-3}$ & $11$B \\
\bottomrule
\end{tabular}
\end{table}

\subsubsection{Splitting the Weights}
We draw special attention to the fact that we split $\mW_{\mathrm{MLP,UP}}$ from $\mW_{\mathrm{MLP,GATE}}$ and orthogonalize them separately.
Ordinary implementations of SwiGLU MLPs concatenate these into a single weight matrix; however, their contributions to the activation are fundamentally different, so gradients are different too.
We find that orthogonalizing them separately improves the final loss; for example, in Llama-430M, we observe an improvement of $\approx 0.2$ in perplexity.
For MoE architectures---whose intermediate size is typically smaller than the hidden size---separating them also reduces the FLOP cost of orthogonalization by a factor of two for standard Newton-Schulz and even more for Gram Newton-Schulz.
Likewise, while earlier implementations of Muon orthogonalized the combined matrix $\begin{bmatrix} \mW_Q \,\vert\, \mW_K \,\vert\, \mW_V\end{bmatrix}$, we orthogonalize each piece separately.

We are also aware that in some settings, including pretraining of GLM-5, Muon benefits from splitting the Multi-Latent Attention weights ($\mW^{UQ}$, $\mW^{UK}$, and $\mW^{UV}$) by attention head before orthogonalizing \citep{glm5team2026glm5vibecodingagentic}.
This choice is principled, since the actual function computed by attention treats each head separately; it does not see the concatenated weights.
Inspired by this, we experimented with splitting $\mW_Q$, $\mW_K$, $\mW_V$, and $\mW_O$ by attention head to form $H$ matrices each of size $\tfrac{d}{H} \times d$, where $d$ is the embedding dimension and $H$ is the number of heads.
However, this design led to higher losses throughout training, so we did not adopt it.

Still, we believe that there are other settings like GLM-5 where this strategy works well.
Such cases would benefit \emph{immensely} from Gram Newton-Schulz, since the aspect ratio of these weight matrices would be the number of heads $H$.
For a standard attention weight like $\mW_Q$ with $H=16$ and $T=5$, Gram Newton-Schulz on the little matrices would use $\mathbf{80\times}$ fewer FLOPs than orthogonalizing the big matrix!

\subsection{Kernelized Gram Newton-Schulz preserves quality}
\label{sec:gram-ns-quality}
We first verify that switching from standard Newton-Schulz to our kernelized implementation of Gram Newton-Schulz does not affect training quality.
We try each implementation with both the Polar Express coefficients \citep{amsel2026the} and those of \citet{cesista2025muonoptcoeffs}.
As \Cref{fig:gram-ns-ppl-hopper} shows, the loss curves are identical in all cases, and the final validation perplexity is preserved to within $0.01$.
\Cref{fig:gram-ns-ppl-hopper} uses a Hopper GPU, but we got the same results on Blackwell.

We see loss preserved as follows, when both using the Polar Express coefficients and the coefficients derived by \citet{cesista2025muonoptcoeffs}:

\begin{figure}[h]
\centering
\includegraphics[width=\textwidth]{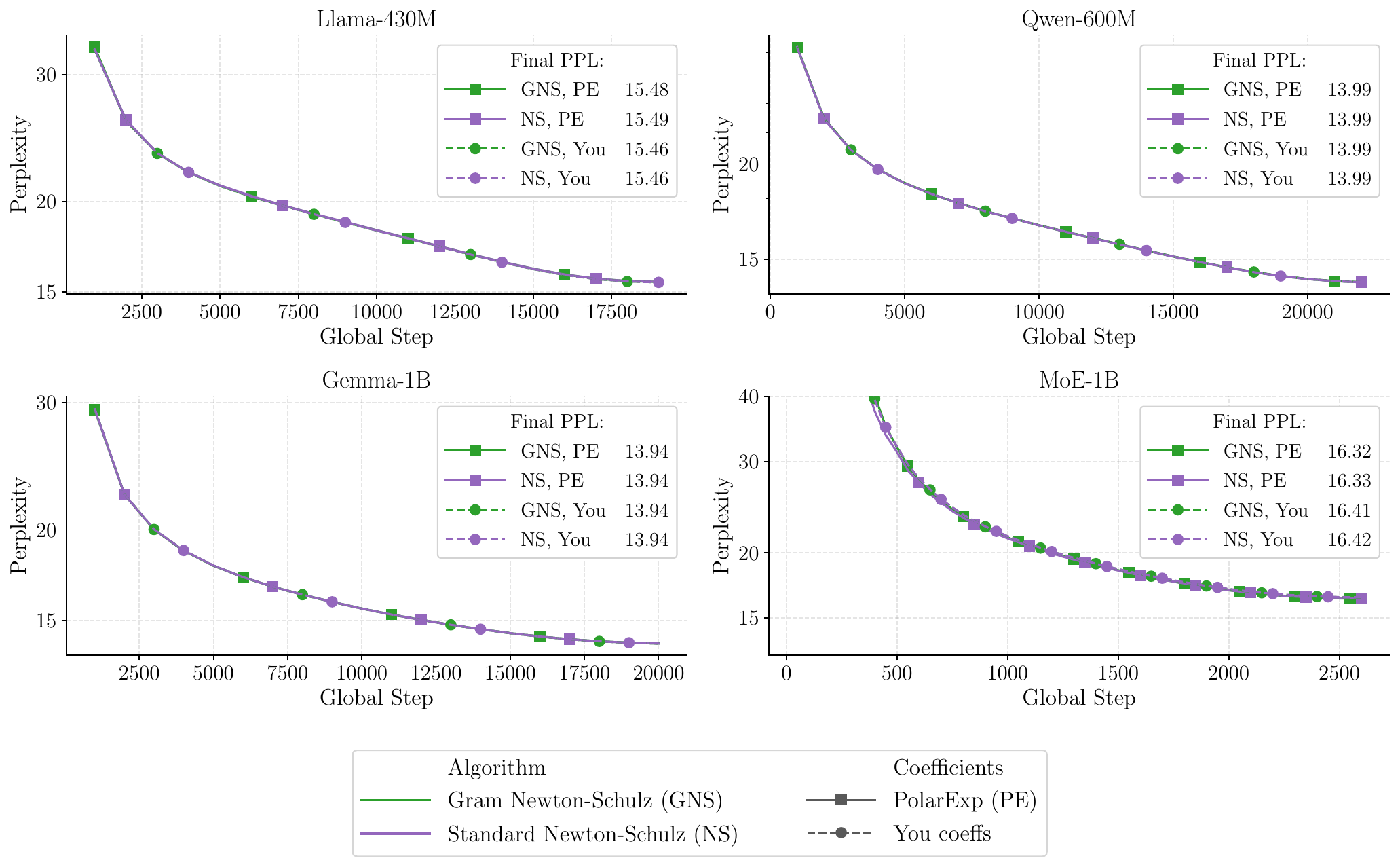}
\caption{When training with Muon on a Hopper GPU, switching from standard Newton-Schulz to Gram Newton-Schulz preserves the validation perplexity throughout training (up to a 0.01 difference in GNS's favor).
Setup is described in \Cref{app:onegpu-setup}.}
\label{fig:gram-ns-ppl-hopper}
\end{figure}

\subsection{Kernelized Gram Newton-Schulz speeds up the optimizer step}

\paragraph{Newton-Schulz Performance}
We observe that our method speeds up the runtime of the Newton-Schulz step by $1.5$--$2\times$.
\Cref{fig:ns-time-per-weight} reports these speed-ups for each model, benchmarked on both H100 and B300 GPUs.
As expected, savings are greatest for weights that are highly rectangular, like Gemma's MLP weights ($\alpha = 8$) and MoE-1B's expert weights ($\alpha = 4$).
Note that these experiments use standard Newton-Schulz as the fallback when $m = n$.

\begin{figure}[p]
\centering
\includegraphics[width=0.85\textwidth]{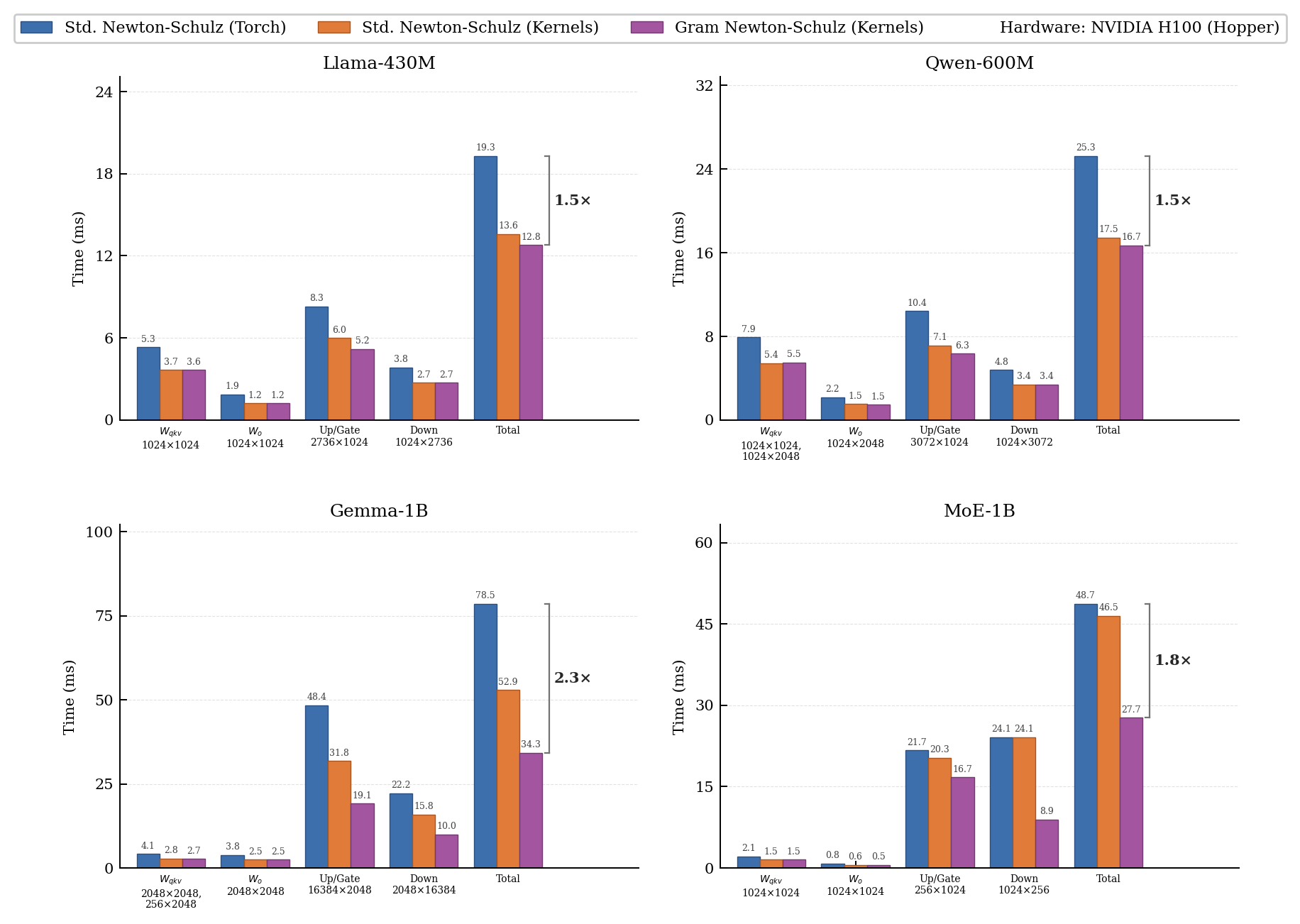}\\
\includegraphics[width=0.85\textwidth]{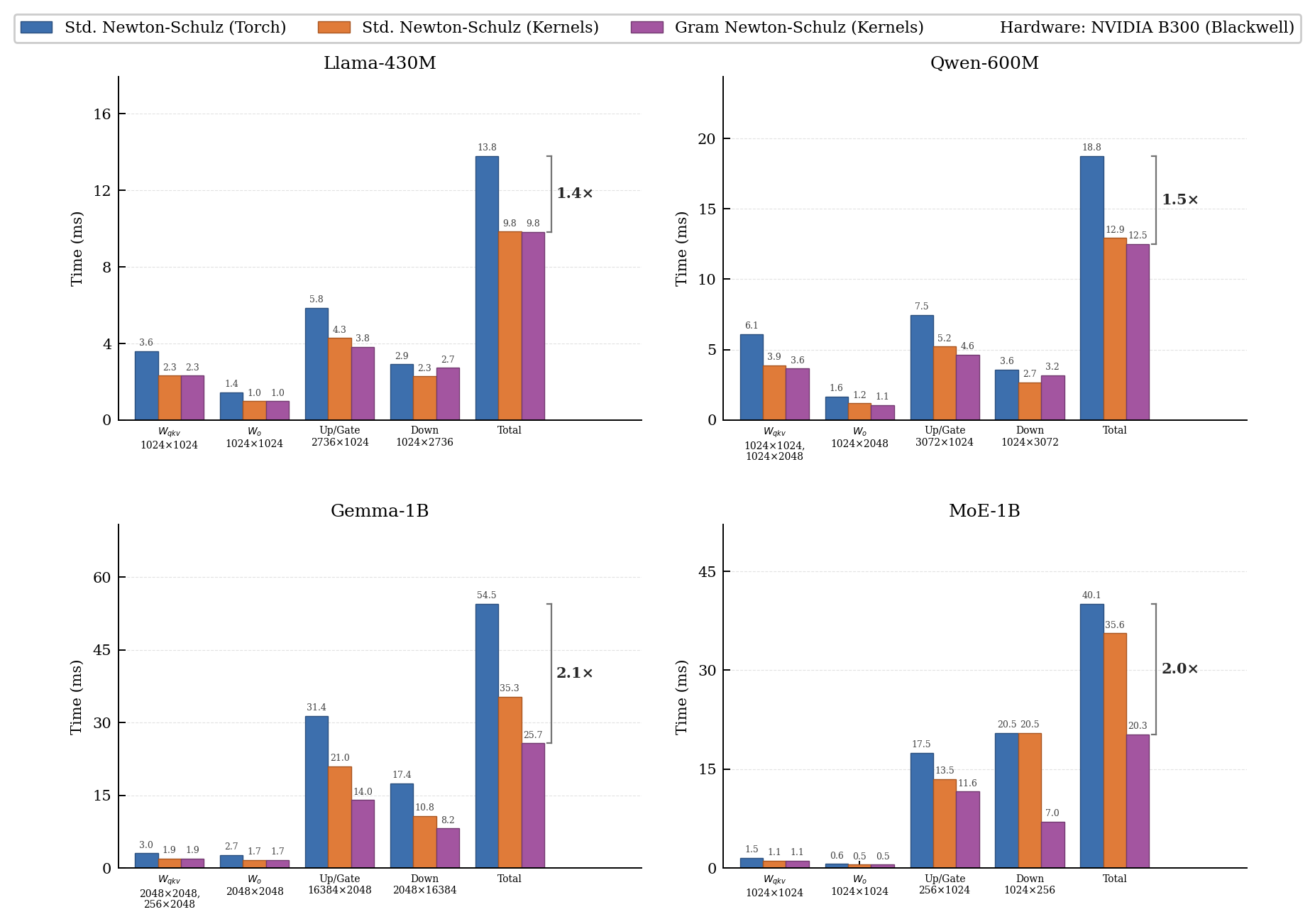}
\caption{Newton-Schulz time per model weight for (1) standard Newton-Schulz implemented in pure PyTorch, (2) standard Newton-Schulz with our symmetric GEMM kernels, and (3) Gram Newton-Schulz with our kernels.
We test four architectures on a single Hopper (top) or Blackwell (bottom) GPU.
Square weights ($\mW_Q$ or $\mW_K$) benefit from the kernels.
Rectangular weights---especially those with a high aspect ratio (e.g.\ Up/Gate in Gemma-1B)---additionally benefit from Gram Newton-Schulz.}
\label{fig:ns-time-per-weight}
\end{figure}

\paragraph{End-to-End Optimizer Performance}
\Cref{fig:optim-step-hopper} shows the end-to-end wall clock time of the Muon optimizer step for each version of Newton-Schulz, along with that of AdamW.
For Muon, these timings include updating the momentum matrix, learning rate scaling, applying the weight update, and the AdamW updates for weights not assigned to Muon (such as the embedding layer and the vector-valued weights).
Our method yields a $1.3$--$2\times$ speedup, with both Gram Newton-Schulz and the kernels contributing significantly.
As before, the impact of Gram Newton-Schulz is greatest for architectures with highly rectangular weights.
Due to its MLP's higher $8\times$ aspect ratio, Gemma-1B sees the largest speedup.

\begin{figure}[h]
\centering
\includegraphics[width=0.75\textwidth]{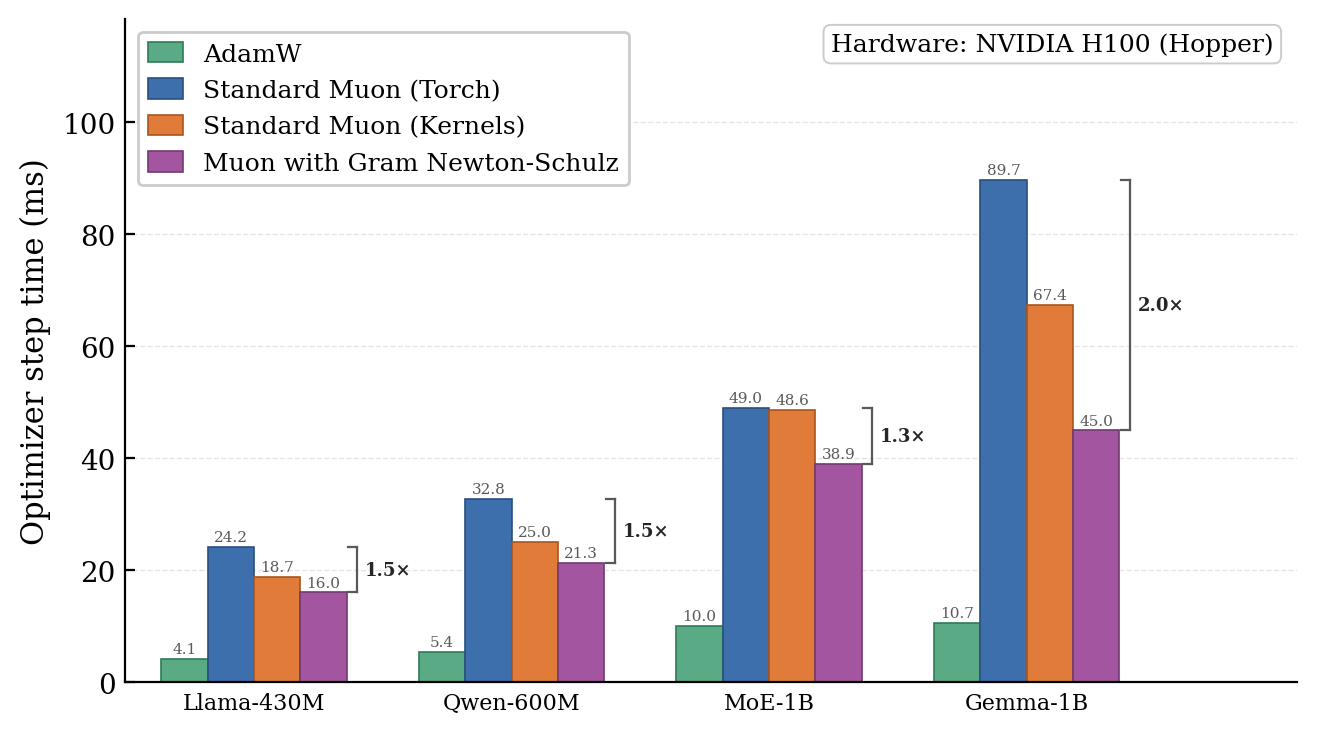}
\caption{Optimizer step time for AdamW and Muon with three different Newton-Schulz routines on a single H100.
Gram Newton-Schulz with our symmetric kernels gives a $1.3$--$2\times$ speedup over standard Newton-Schulz.
Timings include matrix splitting and recombination for QKV and MLP, learning rate scaling, weight updates, and the scalar optimizer (AdamW) step for non-2D weights.}
\label{fig:optim-step-hopper}
\end{figure}

\paragraph{Estimating Gram Newton-Schulz time in Kimi K2}
Kimi K2 is a trillion parameter sparse, fine-grained MoE model with $384$ experts per layer, a hidden size of $7168$, and a small expert intermediate dimension of $2048$.
Since models are trending towards finer-grained MoE architectures and Kimi K2 was trained with Muon, this is a perfect setting to benchmark Gram Newton-Schulz.
Due to Kimi's sophisticated pipeline parallel strategy, the optimizer steps of many of its weights are completely hidden behind the backward pass of the next pipeline stage.
We estimate that orthogonalization steps of following weights are exposed: $216$ expert up/gate/down weights of shape $2048 \times 7168$ and $1$ dense up/gate/down weight of shape $7168\times18432$.
Therefore, we estimate the exposed Newton-Schulz time by orthogonalizing weights of these shapes on a single GPU.
\Cref{fig:kimi-k2-ns} shows the results: our method yields a $2\times$ speedup.
As before, both our kernels and the Gram Newton-Schulz algorithm contribute significantly.

\begin{figure}[ht]
\centering
\includegraphics[width=\textwidth]{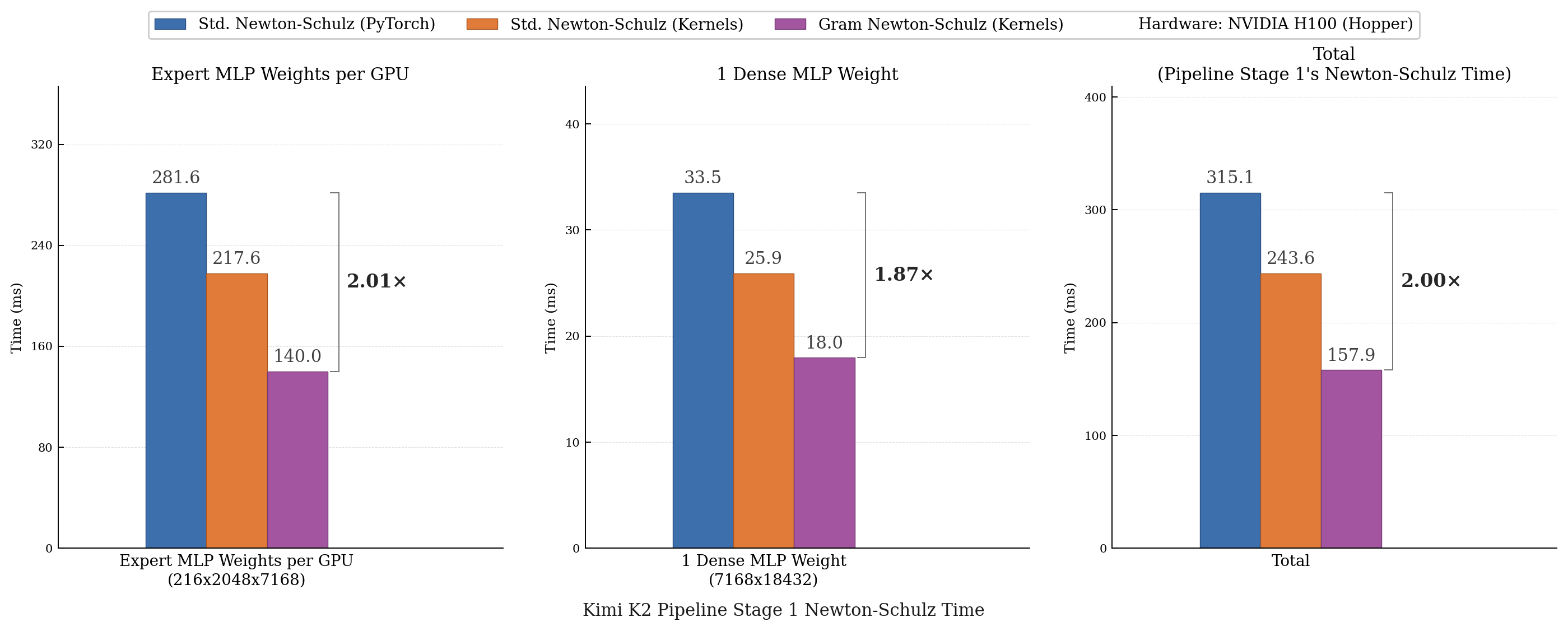}\\
\includegraphics[width=\textwidth]{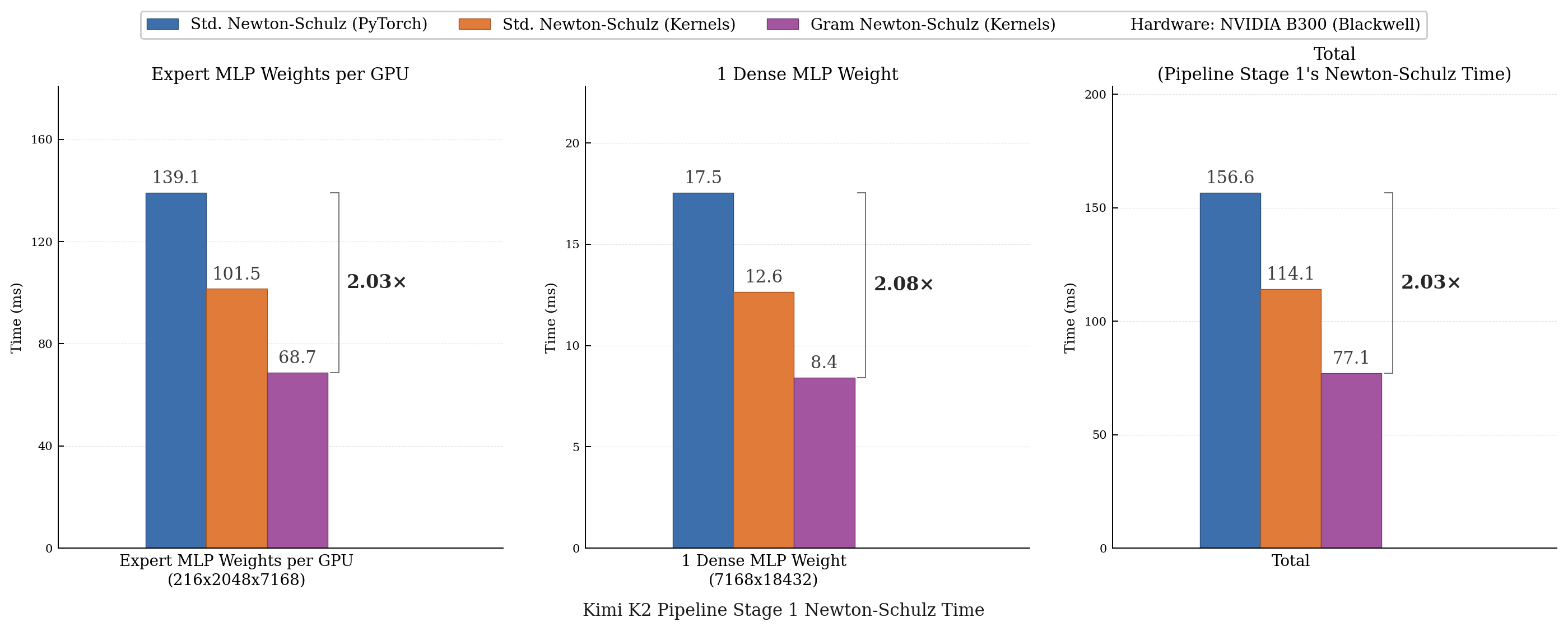}
\caption{Estimated exposed Newton-Schulz time for one step of Kimi K2 with pipeline parallelism.
We measure the runtime of the exposed operations on a single Hopper (top) or Blackwell (bottom) GPU.
Gram Newton-Schulz with kernels is $2\times$ faster than standard Newton-Schulz.}
\label{fig:kimi-k2-ns}
\end{figure}

\section{Stability of Gram Newton-Schulz}
\subsection{Instability of Naive Gram Newton-Schulz}
\label{sec:instability-naive-gram-ns}
We trained Llama-430M with Muon using Naive Gram Newton-Schulz (\Cref{alg:naive-gram-ns}).
The loss curve is shown in \Cref{fig:llama_430_no_reset}.
Clearly, training is very unstable.
Not only do we observe loss spikes, but eventually, the outputs of Gram Newton-Schulz become full of Infs.
The problem is due to floating point arithmetic.
While Gram Newton-Schulz is mathematically equivalent to standard Newton-Schulz in exact arithmetic, it behaves differently in finite precision, especially in half precision.
As half precision is essential for good performance, this makes Naive Gram Newton-Schulz unworkable in practice.

We now analyze the source of the instability in order to motivate our solution.
A Jupyter notebook reproducing the experiments in this section is available \href{https://github.com/NoahAmsel/PolarExpress/blob/appF-stability/gram_newton_schulz_stability.ipynb}{here}.

\begin{figure}[ht]
\centering
\includegraphics[width=0.6\textwidth]{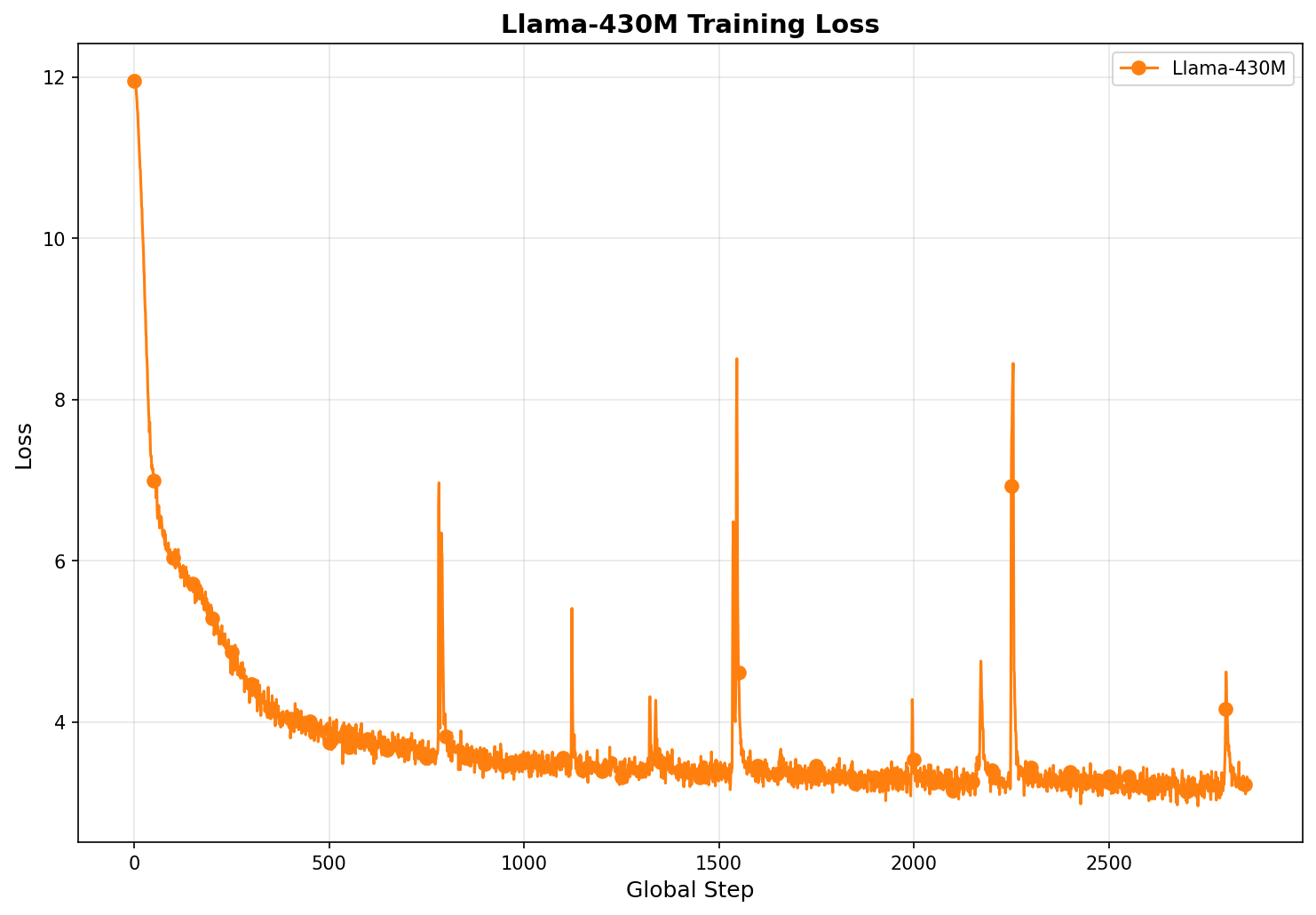}
\caption{Naive Gram Newton-Schulz used to train Llama-430M in half precision.
Numerical instability wrecks training.}
\label{fig:llama_430_no_reset}
\end{figure}

\subsubsection{Tracking Eigenvalues of Intermediate Matrices}
To understand how the matrices in Naive Gram Newton-Schulz evolve and why they diverge, we track their eigenvalues and singular values.
Recall that the entries of any matrix are upper bounded by its largest singular value, so if we can bound the singular values, we can prevent blowups.

As a baseline, we start by running Naive Gram Newton-Schulz in full \texttt{float64} precision for $8$ steps to simulate its behavior in exact arithmetic.
We use a synthetic $128 \times 512$ matrix with an exponentially decaying spectrum.
To make our plots more readable, the experiments in this section use the coefficients $(a_t, b_t, c_t) = (\tfrac{15}{8}, -\tfrac{10}{8}, \tfrac{3}{8})$ at every iteration, but our conclusions will generalize to the coefficients used in practice.
If $\mX_0 = \mU \mSigma \mV^\top$ is the SVD of the input matrix, then the intermediate matrices of \Cref{alg:naive-gram-ns} ($\mR_t$, $\mQ_t$, $\mZ_t$) are square symmetric with eigenvectors $\mU$.
We therefore plot the diagonal entries of $\mU^\top \mR_t \mU$ and $\mU^\top \mQ_t \mU$ against the corresponding singular values in $\mSigma$ to track how each evolves according to the polynomial update rules---or diverges from them.
Even though Gram Newton-Schulz does not need to compute $\mX_1, \ldots, \mX_{T-1}$, we do so here for demonstration using the formula $\mX_t = \mQ_t \mX_0$ and plot $\mU^\top \mX_t \mV$ against $\mSigma$.
The top panel of \Cref{fig:gns_diagnostics} shows that the eigenvalues evolve as predicted by \Cref{thm:gram-ns}.
Initially, $x_0 \in [0, 1]$, $r_0 = x_0^2 \in [0, 1]$ and $q_0 = 1$.
As the algorithm progresses, $r_t \to 1$, $q_t = x_t / x_0 \to 1/x_0$, and $x_t \to 1$.
For $x_0 \approx 1$, convergence is faster; for $x_0 \ll 1$, it is slower.

\begin{figure}[H]
\centering
\includegraphics[width=\textwidth]{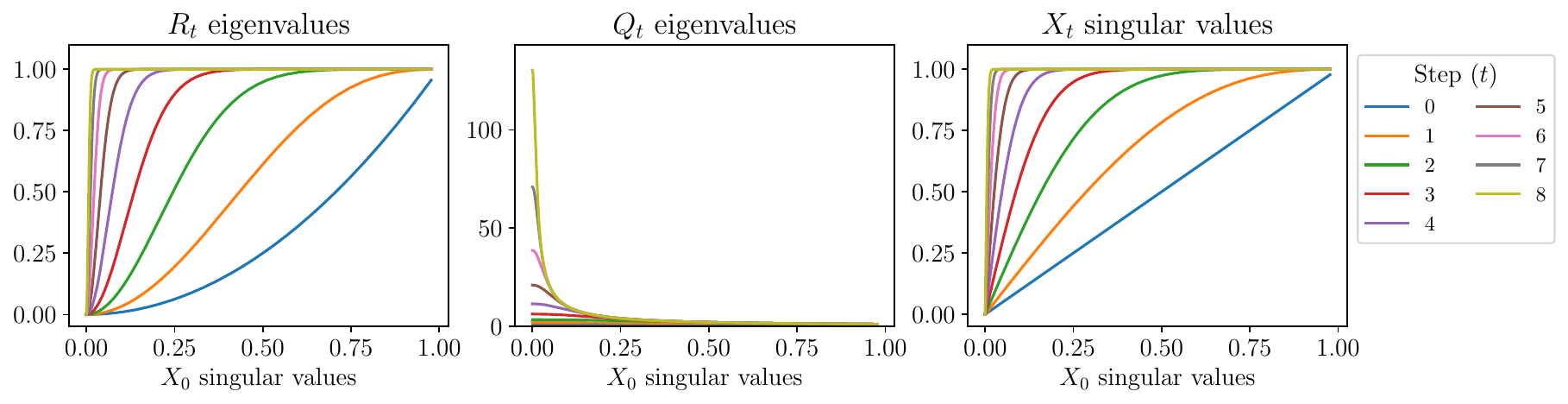}\\[1ex]
\includegraphics[width=\textwidth]{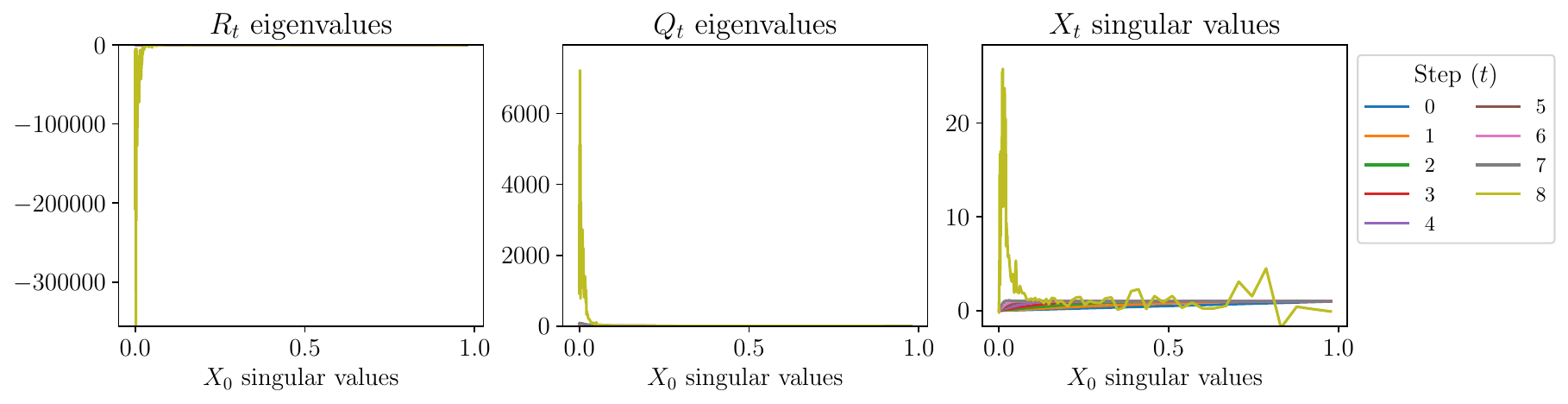}
\caption{Evolution of eigenvalues of $\mR_t$, $\mQ_t$, and $\mX_t$ in Naive Gram Newton-Schulz with $(a_t, b_t, c_t) = (\nicefrac{15}{8}, \nicefrac{10}{8}, \nicefrac{3}{8})$ as a function of the corresponding singular value of $\mX_0$. \textbf{Top:} \texttt{float64}. \textbf{Bottom:} \texttt{bfloat16}.}
\label{fig:gns_diagnostics}
\end{figure}

Now we rerun the experiment in \texttt{bfloat16} arithmetic (\Cref{fig:gns_diagnostics}, bottom).
The first few iterations look correct, but by step 7, we see noisy and unexpected behavior ($r_t < 0$, $x_t > 1$).
By step 8, the spectrum diverges completely, and by step 10 the algorithm returns Infs.
We now describe two key causes of divergence: spurious negative eigenvalues in the Gram matrix $\mX \mX^\top$, and eigenvector drift.

\subsubsection{Spurious Negative Eigenvalues}
The leading cause of divergence is the presence of negative eigenvalues in the Gram matrix due to half-precision arithmetic.
These negative eigenvalues blow up after too many iterations of Gram Newton-Schulz.

By construction, $r_t = x_t^2 \geq 0$, so $\mR_t$ should be positive semidefinite (\Cref{thm:gram-ns}), but \Cref{fig:gns_diagnostics} shows that is has negative eigenvalues.
In fact, even $\mR_0$ has tiny negative eigenvalues introduced in the first matrix multiplication $\mX_0 \mX_0^\top$.
Because $\mX_0$ has many singular values that are nearly zero (as is the case in Muon), $\mR_0 = \mX_0 \mX_0^\top$ has many eigenvalues that are \emph{numerically equal} to zero.
In \texttt{bfloat16}, even a slightly negative number can be numerically equal to zero.
Later computations can introduce additional negative eigenvalues into $\mR_t$ as well.
These eigenvalues represent nothing about the original problem; they are just an artifact of floating point arithmetic.
Therefore, we call them ``spurious eigenvalues''.

These spurious negative eigenvalues start small, but \Cref{fig:gns_diagnostics} shows that their magnitude grows quickly.
Recall the update rule:
\begin{equation}\label{eq:r_update_rule}
r_t = r_{t-1} z_t^2 = r_{t-1} h_t(r_{t-1})^2
\end{equation}
For our choice of coefficients, $h_t(x) = \tfrac{15}{8} - \tfrac{10}{8} x + \tfrac{3}{8} x^2$.
\Cref{fig:r_update_shrunk} plots this update rule.
As it shows, $r_t < \left(\tfrac{15}{8}\right)^2 r_{t-1}$.
Thus, if any $r_t < 0$, the spurious negative eigenvalues grow exponentially, diverging to $-\infty$ !
This sets off a chain reaction that causes $\mQ_t$ and $\mX_t$ to diverge as well.
This problem cannot be fixed by choosing different polynomials; while the main loop of Gram Newton-Schulz approximates the inverse square root of $r_0 > 0$, it diverges for negative inputs.

\begin{figure}[h]
\centering
\includegraphics[width=0.4\textwidth]{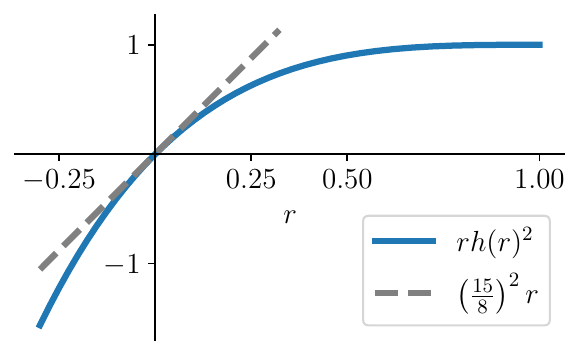}
\caption{When $\mR_t$ evolves according to \eqref{eq:r_update_rule}, negative eigenvalues diverge to $-\infty$.}
\label{fig:r_update_shrunk}
\end{figure}

To prove that the tiny spurious negative eigenvalues of $\mR_0$ suffice to cause divergence, we rerun the experiment in \texttt{float64} precision, but cast $\mR_0$ from \texttt{float64} to \texttt{bfloat16} and then back to \texttt{float64} to induce a small floating point error.
As \Cref{fig:posthoc_f16_diagnostics} shows, this is enough to cause a blowup.

\begin{figure}[h]
\centering
\includegraphics[width=\textwidth]{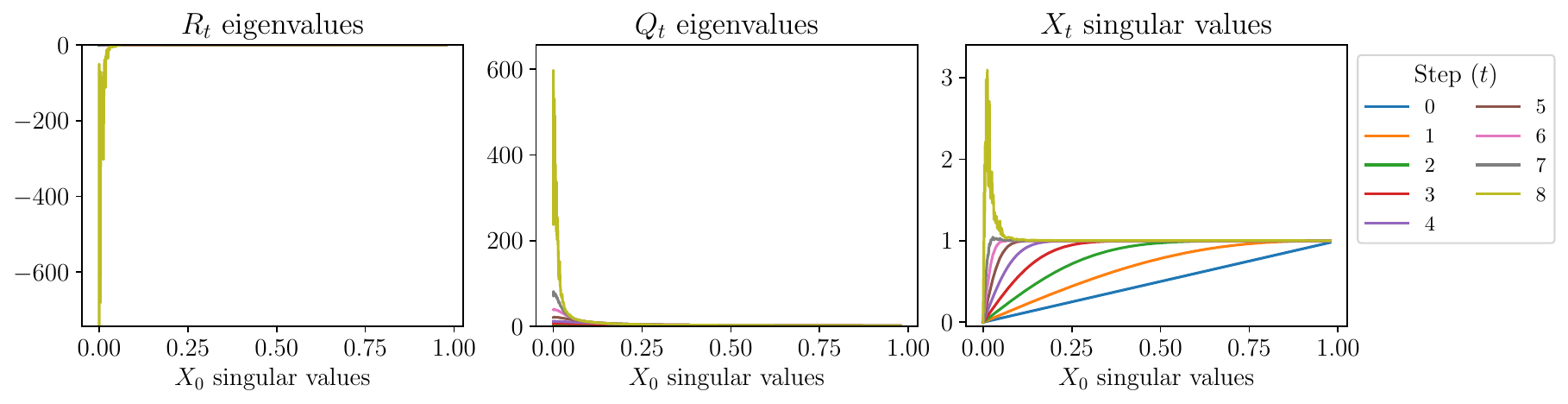}
\caption{Evolution of eigenvalues of $\mR_t$, $\mQ_t$, and $\mX_t$ when all operations use \texttt{float64} except $\mR_0 = \mX_0 \mX_0^\top$, which uses \texttt{bfloat16}.}
\label{fig:posthoc_f16_diagnostics}
\end{figure}

\subsubsection{Eigenvector Drift}
Spurious negative eigenvalues are not the only source of numerical instability.
If we take the input matrix $\mX_0$ to have no small singular values (i.e., all $\geq 0.017$), then we do not observe any negative eigenvalues in $\mR_t$, but $\mX_t$ still fails to converge.
The culprit seems to be ``eigenvector drift''.

In exact arithmetic, the eigenvectors of all intermediate matrices match $\mU$, the left singular vectors of $\mX_0$, but in finite precision they do not.
We demonstrate this effect by measuring how far $\mU^\top \mR_t \mU$, $\mU^\top \mQ_t \mU$, and $\mU^\top \mX_t \mV$ are from being diagonal matrices.
That is, we take the Frobenius norm of the off-diagonal entries as a fraction of the total Frobenius norm.
\Cref{fig:easy_spectrum_diagnostics} shows that after several iterations, the eigenvectors / singular vectors of $\mQ_t$ and $\mX_t$ have drifted significantly from those of $\mX_0$.
At the same time, the eigen\emph{values} of $\mQ_t$ (and by extension, those of $\mX_t$) diverge from where they would be in exact arithmetic.
The growing eigenvalues of $\mQ_t$ seem to spill into one another.
The strength of this effect is less consistent than that of negative eigenvalues, but it is still harmful.

\begin{figure}[h]
\centering
\includegraphics[width=\textwidth]{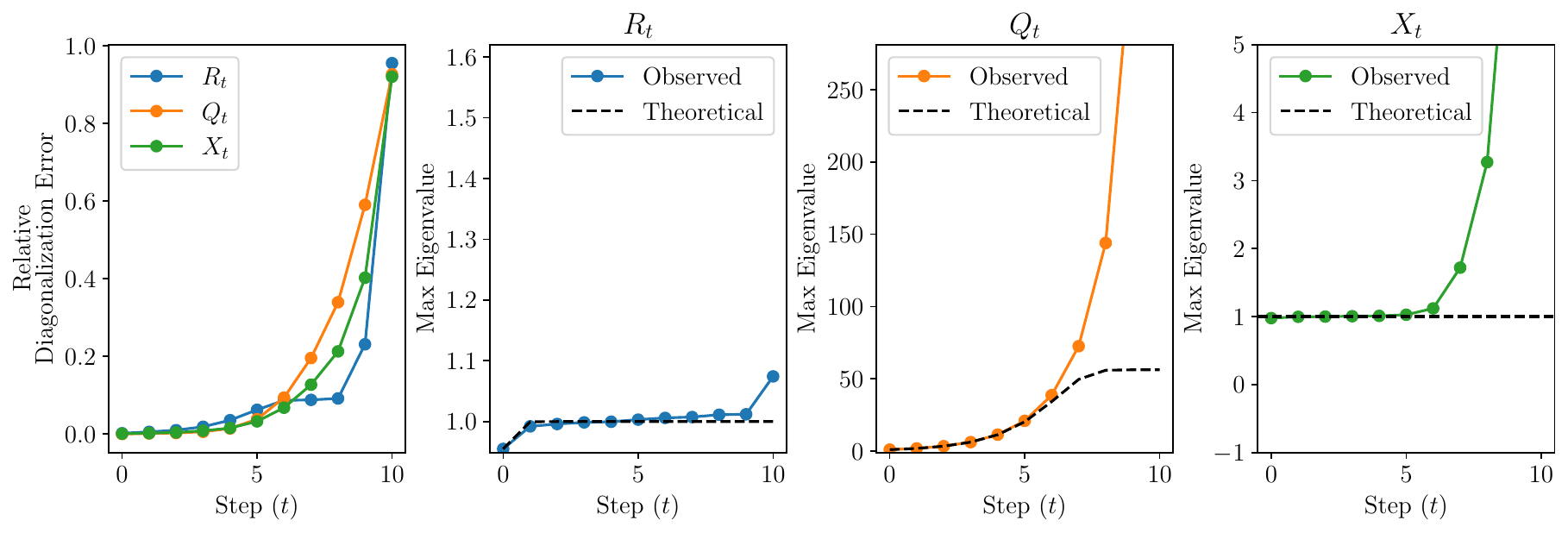}
\caption{As the eigenvectors drift, the spectral norms of $\mR_t$, $\mQ_t$, and $\mX_t$ diverge.}
\label{fig:easy_spectrum_diagnostics}
\end{figure}

\subsection{Stabilizing Gram Newton-Schulz by Restarting}
\label{sec:stabilizing-gram-ns}
As we have seen, if we run Gram Newton-Schulz for more than a few iterations, the spurious negative eigenvalues of $\mR_t$ diverge to negative infinity and $\mQ_t$ blows up.
Our solution is simple: run Gram Newton-Schulz for only a few iterations.
For instance, rather than using Gram Newton-Schulz to compute $\mX_T$ directly, we can use it to compute $\mX_5$ in a stable manner.
While $\mX_5$ is not a good approximation to $\operatorname{polar}(\mX_0)$, it is closer than where we started.
Now we apply Gram Newton-Schulz a second time on the input $\mX_5$ to compute $\mX_{10}$ stably.
This process can be repeated to reach any desired $\mX_T$.
This restarting technique sacrifices some of the performance gains of Gram Newton-Schulz, but it still offers a significant speedup over standard Newton-Schulz.

\Cref{fig:restart5_diagnostics} repeats the experiment from above with a restart every five iterations.
While $\mR_t$ develops some negative eigenvalues, unlike before, the growth of these eigenvalues is controlled.
Each time we restart, we re-initialize $\mR_t = \mX_t \mX_t^\top$, eliminating any large negative eigenvalues.
By restarting regularly, we prevent any from growing too large.
As expected, $\mQ_t$ resets to the identity at iterations $5, 10, 15, 20$, and $25$.
Therefore, the eigenvalues of $\mQ_t$ never grow beyond $\approx 12$, despite the negative eigenvalues in $\mR_t$.
Since the eigenvalues of $\mQ_t$ remain controlled, those of $\mX_t = \mQ_t \mX_{t-5}$ stay at or below $1$.

\begin{figure}[ht]
\centering
\includegraphics[width=\textwidth]{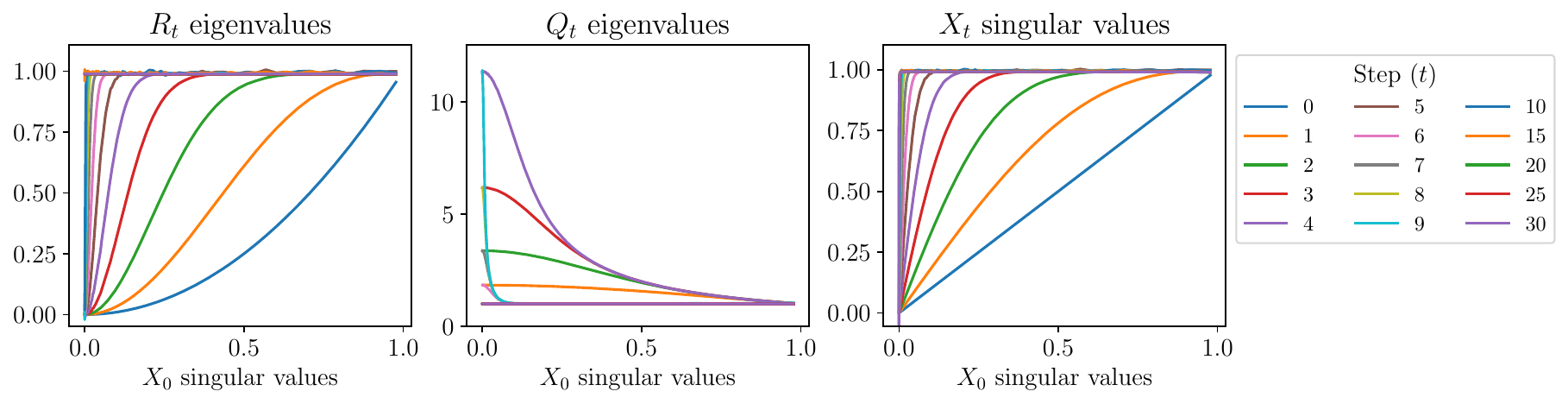}
\caption{Restarting prevents the divergence of $\mR_t$ in half-precision.}
\label{fig:restart5_diagnostics}
\end{figure}

Restarting also helps control eigenvector drift.
We repeat the experiment from \Cref{fig:easy_spectrum_diagnostics} on the same matrix (with all singular values $> 0.017$) with a restart after step $5$.
Diagonalization error always remains $\leq 0.05$, and the maximum eigenvalues now closely track their predicted values.
Note that we always measure eigenvector drift relative to the original input $\mX_0$, not the restarted $\mX_5$.

\begin{figure}[ht]
\centering
\includegraphics[width=\textwidth]{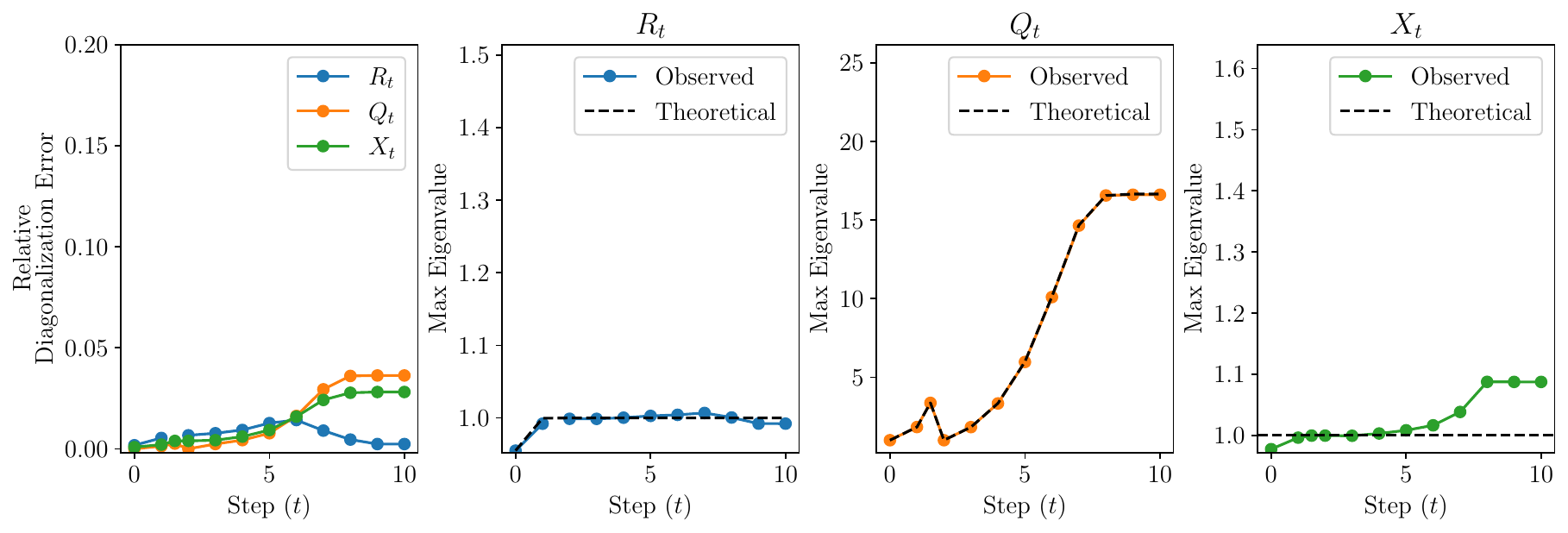}
\caption{Restarting curbs eigenvector drift.}
\label{fig:easy_spectrum_restart2_diagnostics}
\end{figure}

\subsubsection{When to Restart: Polar Express Coefficients for Muon}
We now derive the optimal restarting schedule for a given total number of iterations $T$.
To prevent $\mX_t = \mQ_t \mX_0$ from blowing up, we must control the condition number of $\mQ_t$, even when $\mR_0$ has spurious negative eigenvalues.
(Because $q_t \to 1/x_0 \geq 1$, this also controls the maximum eigenvalue of $\mQ_t$.)
The growth of $\mQ_t$ in turn depends on the size of the spurious negative eigenvalues and the specific sequence of polynomials defined by $\{(a_t, b_t, c_t)\}_{t=1}^T$.
Fixing a desired number of restarts, we sweep over all possible restart schedules and pick the one that minimizes the maximum condition number of $\mQ_t$ across $t$s.

For the application to Muon, we analyze five iterations of the Polar Express coefficients \citep{amsel2026the}.
In the experiments of the previous section, we observe that the most negative spurious eigenvalue of $\mR_0$ is about $-4 \cdot 10^{-4}$.
Therefore, we simulate Gram Newton-Schulz in full precision with Polar Express coefficients and track how the eigenvalues of $\mR_t$ and $\mQ_t$ evolve when $\mR_0$ has eigenvalues in the range $[-4\cdot 10^{-4}, 1]$.
The left panels of \Cref{fig:polar_restart_sweep} show that without restarting, they blow up.
We then repeat the simulation with one restart, sweeping all possible choices of when to restart.
Every time we restart and form $\mR = \mX\mX^\top$, we subtract $4 \cdot 10^{-4}\mI$ to simulate potentially dangerous corruption of the eigenvalues due to floating point error.
As the right panel shows, restarting after the \emph{second} iteration provides the best bounds, ensuring that the eigenvalues of $\mR_t$ stay well above $-0.4$ and those of $\mQ_t$ stay below $\approx 100$ for all iterations.
\Cref{fig:final_diagnostics} shows that for Gram Newton-Schulz with Polar Express coefficients and one restart after the second iteration, all eigenvalues of $\mX_t$ converge stably to $1$ .

\begin{figure}[h]
\centering
\begin{minipage}[b]{0.48\textwidth}
\includegraphics[width=\linewidth]{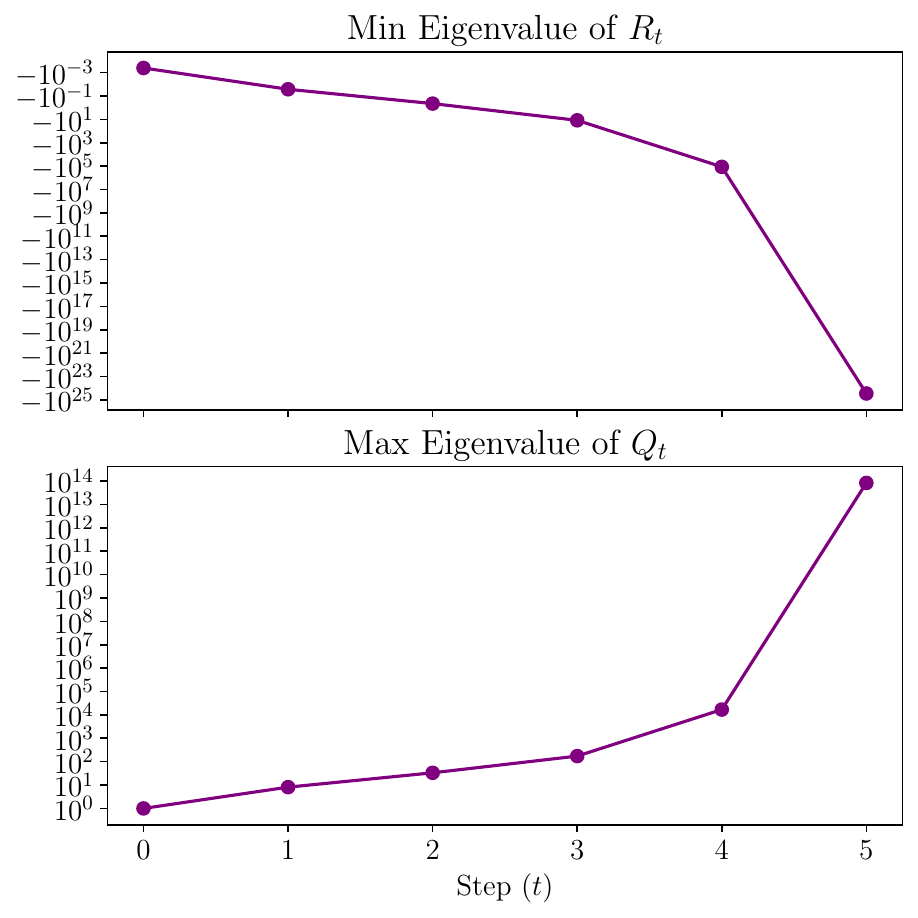}
\end{minipage}
\hfill
\begin{minipage}[b]{0.48\textwidth}
\includegraphics[width=\linewidth]{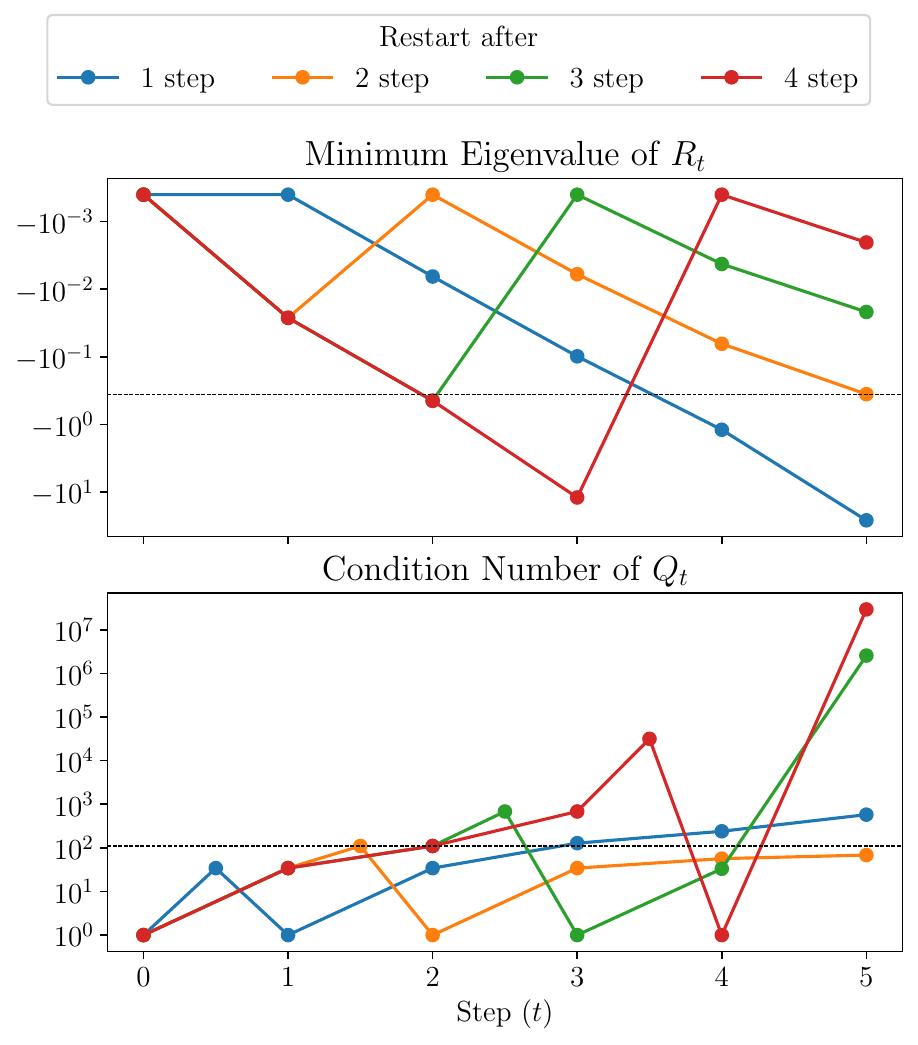}
\end{minipage}
\caption{\textbf{Left:} Min/max eigenvalue of $\mR_t$ and $\mQ_t$ without restarts.
$\mR_0$ starts with a negative eigenvalue at $-4 \cdot 10^{-4}$.
\textbf{Right:} Minimum eigenvalue of $\mR_t$ and condition number of $\mQ_t$ if restart is placed after iteration $1$, $2$, $3$, or $4$.
Restarting after iteration $2$ best controls the size of $\mQ_t$.}
\label{fig:polar_restart_sweep}
\end{figure}

\begin{figure}[H]
\centering
\includegraphics[width=\textwidth]{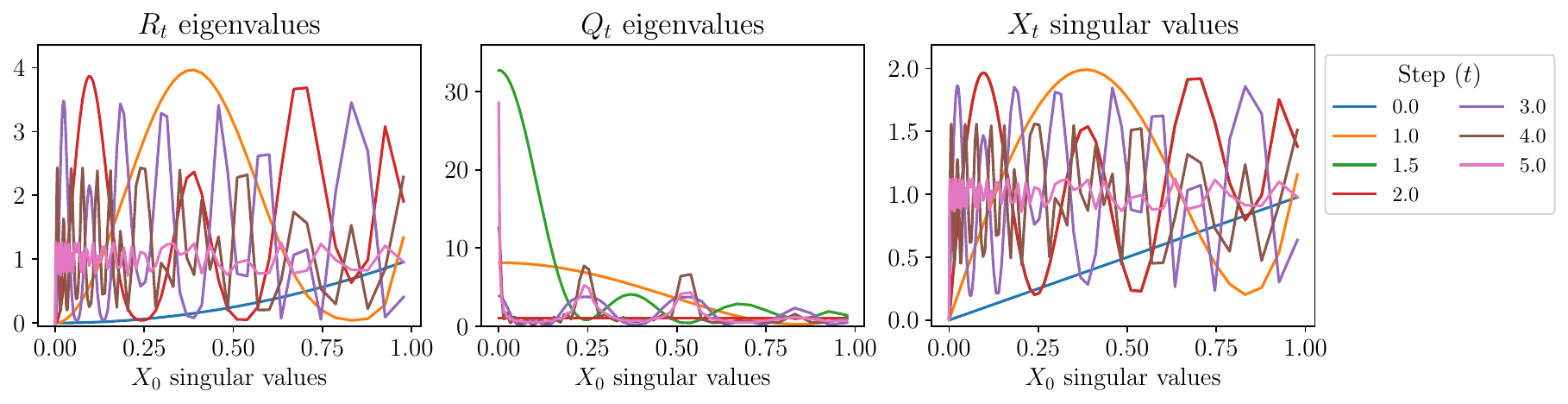}
\caption{Gram Newton-Schulz with Polar Express coefficients and a restart after $2$ iterations converges stably.}
\label{fig:final_diagnostics}
\end{figure}

\subsubsection{Further Precautions}
While restarting greatly improves stability, it is not absolutely foolproof.
A second or third restart may be required if running for more than five iterations or using a numerically sensitive set of coefficients.
Moreover, the usual numerical precautions for standard Newton-Schulz still apply.

\paragraph{Safety Factors}
Most choices of Newton-Schulz polynomials are designed to converge only when the input lies in $[0,1]$; any singular values larger than $1$ may diverge rapidly, as \Cref{fig:X_final_unbounded} shows.
Even when $\mX_0$ is properly normalized, singular values greater than $1$ can arise due to numerical error.
This problem affects standard Newton-Schulz too, so the Polar Express polynomials are typically adjusted according to the formula $\tilde p_t(x) = p_t(x / 1.02)$.
This ensures convergence even for singular values as large as $1.02$.
When using Gram Newton-Schulz, roundoff errors like this can worsen due to computations like $\mX\mX^\top$, which do not have such numerical buffers; however, we have never seen this happen when using our recommended setup (\texttt{float16} arithmetic with restarting after $2$ iterations).
It is wise to be conservative in the choice of safety factor, for instance, by replacing $1.02$ with $1.05$.

\begin{figure}[h]
\centering
\includegraphics[width=0.8\textwidth]{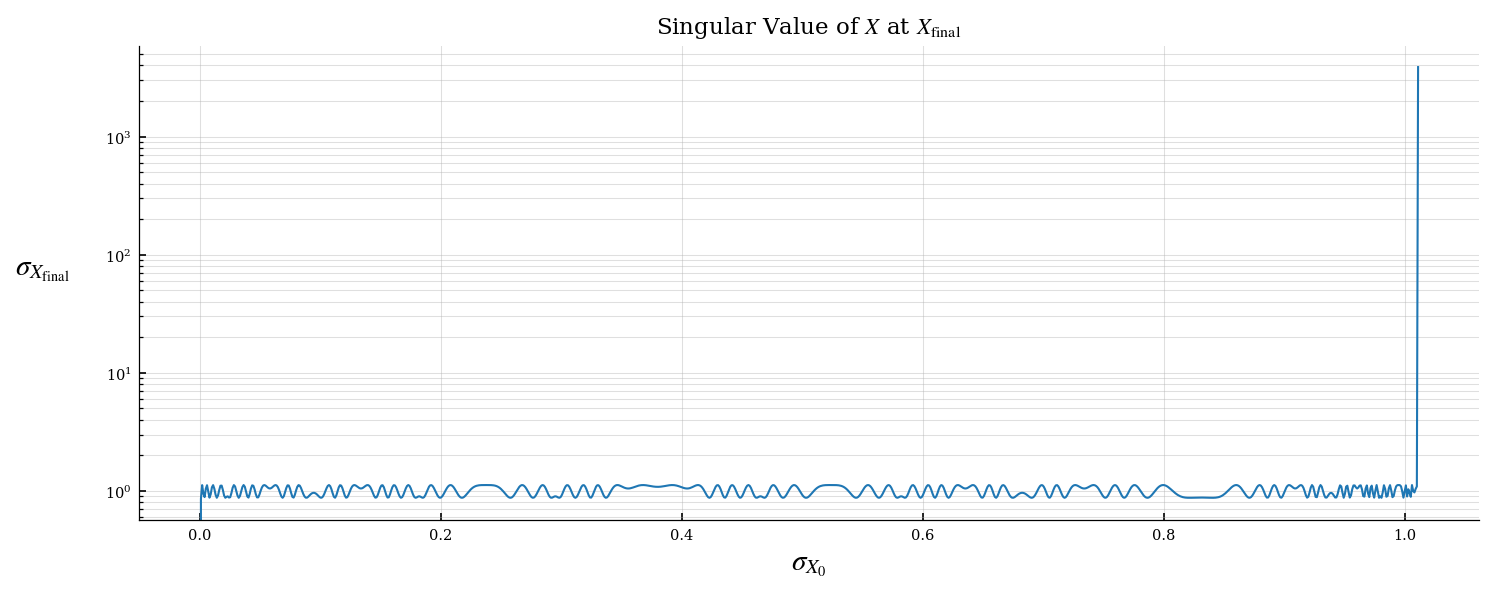}
\caption{Theoretical behavior of both standard and Gram Newton-Schulz (Polar Express coefficients) when $\mX_0$ has singular values slightly above one.}
\label{fig:X_final_unbounded}
\end{figure}

\paragraph{Float16 vs BFloat16 in Newton-Schulz}
\label{sec:float16-vs-bfloat16}
In addition, we argue for using \texttt{float16} instead of \texttt{bfloat16} to implement Newton-Schulz.
Both use $16$ bits and thus run with the same performance; however, the distribution of the $16$ bits across the mantissa and exponent bits differs.
Compared to \texttt{bfloat16}, \texttt{float16} represents values from a narrower range, but it has greater precision within that range.
For Newton-Schulz, however, the range of \texttt{float16} (roughly $6.1\cdot 10^{-5}$ to $6.5 \cdot 10^4$) suffices because the magnitudes of its intermediate matrices are not too large.
On certain test matrices, we see more accurate $\operatorname{polar}(\mX)$ approximations with \texttt{float16}, but in practice, we have not found a case where the training quality is meaningfully different between \texttt{float16} and \texttt{bfloat16}.
Still, we default to \texttt{float16}.

\paragraph{High Accuracy Setting}
In the application to Muon, we do not need to compute $\operatorname{polar}(\mX)$ to high accuracy, and \Cref{sec:gram-ns-quality} shows that Muon with Gram Newton-Schulz yields effectively identical results to Muon with standard Newton-Schulz in terms of training quality.
However, when high accuracy is desired, the usual warnings about forming the Gram matrix apply.
Since forming $\mX \mX^\top$ immediately squares the condition number, Gram Newton-Schulz may not be appropriate in these cases.

\subsubsection{Computing Matrix Quadratics}
\label{sec:computing-matrix-quadratics}
A key step in Gram Newton-Schulz is computing the matrix quadratic $\mZ_t \gets a_t\mI + b_t \mR_{t-1} + c_t \mR_{t-1}^2$.
Standard Newton-Schulz implicitly computes matrix quadratics too, but PyTorch implementations typically avoid adding the identity explicitly.
Rather, they compute $\mX(a_t \mI + b_t \mA + c_t \mA^2)$ in two steps, each with a single GEMM:
\begin{enumerate}
\item $\mB \gets b_t \mA + c_t \mA^2$
\item $\mX \gets a_t \mX + \mB \mX$
\end{enumerate}

Our symmetric GEMM kernel can compute matrix quadratics in one launch, fusing the addition of $a_t \mI$ by adding $a_t$ to all diagonal entries of the output when they are at the register level.
This optimization completely obviates any I/O operations needed for the $a_t \mI$ addition, typically outspeeding \texttt{gemm\_symmetric(A, B, C, c\_t, b\_t) + a\_t * I}, which would require loading $\mI$ from general memory to shared memory to registers.
Once $\mZ_t$ is assembled, Gram Newton-Schulz can perform its three subsequent multiplications without the need for further fused additions.

However, our tests show that adding $a_t \mI$ explicitly can be less stable than distributing it into those later three multiplications.
If we stress-test our method by ignoring some of our own advice---restarting after three iterations instead of two, using a Polar Express safety factor of $1.02$ instead of $1.05$, and computing the quadratic with $a_t \mI$ explicitly---we observe instability.
Interestingly, this instability disappears if we use non-symmetric GEMMs (either from PyTorch or Quack) instead of our symmetric kernels; however, if we force symmetry after calling standard PyTorch GEMMs, we see instability again.
We conclude that fusing $+ a_t \mI$ into a symmetric GEMM is numerically unfavorable; this is not just a kernel bug.

We believe this effect can be explained as follows.
While the fused kernel computes $a_t\mI + b_t \mR_{t-1} + c_t \mR_{t-1}^2$ in \texttt{float32} arithmetic under the hood, the result $\mZ_t$ is rounded back down to \texttt{float16} at the end of the GEMM.
Future computations like $\mQ_t \mZ_t$ suffer from this loss of precision in $a_t$.
In contrast, if the $a_t \mI$ term is handled implicitly, all arithmetic involving $a_t$ takes place in \texttt{float32}.
Therefore, it is more stable to compute $a_t \mQ_t + \mQ_t\left(b_t \mR_{t-1} + c_t \mR_{t-1}^2\right)$ than $\mQ_t\left(a_t\mI + b_t \mR_{t-1} + c_t \mR_{t-1}^2\right)$.

We reiterate that in all our experiments, this instability can be avoided entirely by restarting correctly or using a safety factor of $1.05$.
Out of an abundance of caution, we rearrange Naive Gram Newton-Schulz (\Cref{alg:naive-gram-ns}) to avoid adding $a_t \mI$ explicitly.
We change
\begin{enumerate}
\item $\mZ_t \gets a_t\mI + b_t \mR_{t-1} + c_t \mR_{t-1}^2$ \qquad \qquad // Apply $h_t(\mR_{t-1})$
\item $\mQ_t \gets \mQ_{t-1} \mZ_t$
\item $\mR_t \gets \mZ_t \mR_{t-1} \mZ_t$
\end{enumerate}
to
\begin{enumerate}
\item $\mZ_t \gets b_t \mR_{t-1} + c_t \mR_{t-1}^2$
\item $\mQ_t \gets \mQ_{t-1} \mZ_t +  a_t\mQ_{t-1}$
\item $(\mathbf{RZ})_t \gets \mR_{t-1} \mZ_t + a_t\mR_{t-1}$
\item $\mR_t \gets \mZ_t (\mathbf{RZ})_t + a_t(\mathbf{RZ})_t$
\end{enumerate}

This change fixes all collected training examples in which symmetric GEMMs were less stable than non-symmetric GEMMs.

\section{Kernel Implementation Details}
\label{app:kernel-details}
\subsection{Symmetric GEMM Kernel Details}
\label{sec:symmetric-gemm-kernel-details}
We implement all of our symmetric GEMM kernels with square cluster work tiles.
Hopper uses cluster size $(2, 1)$ and thread block tile size $(128, 256)$, and Blackwell uses cluster size $(2, 1)$ and 2-CTA collaboration, in which the 2 thread blocks in the cluster collaborate on the same big $(256, 256)$ tile.
Notably, highly optimized custom GEMM kernels on Hopper typically use Ping Pong Scheduling, in which the MMA of tile $i$ and the epilogue of tile $i-1$ are overlapped in two consumer warp groups.\footnote{\url{https://pytorch.org/blog/cutlass-ping-pong-gemm-kernel/}}
However, Ping Pong Scheduling uses more registers at once, and $(128, 256)$ is too large of a tile size for Ping Pong Scheduling, leading to register spillage.
This is much slower than standard single producer warp, single consumer warp scheduling.
Thus, our Hopper symmetric kernels do not use Ping Pong Scheduling.
Blackwell GEMM kernels have no explicit conception of Ping Pong Scheduling, since by default in cuBLAS and most kernel libraries, two accumulators are kept in the new tensor memory hierarchy, and MMA is computed on one accumulator while the epilogue is computed on the other.

As a small implementation detail, note that the main diagonal of $256 \times 256$ cluster work tiles is part of the work assigned by the triangular scheduler.
Since their transposed locations are identical to their current locations, we only write those values to general memory once---writing twice can cause inaccurate values or NaNs.

\subsection{Implementation Strategy in Code}
\label{sec:implementation-strategy-in-code}
There are only two differences between the symmetric GEMM kernel and the standard GEMM kernel: the triangular scheduler and the transposed tile write in the epilogue.
We design our symmetric kernel such that it is abstracted around the standard GEMM kernel to enable lightweight but maximally performant GEMM epilogue fusions.
Using these abstractions, we are able to implement the symmetric GEMM kernel for both Hopper and Blackwell in just 160 lines, while achieving state-of-the-art performance.

We override the standard tile scheduler with our triangular scheduler and wrap the symmetric GEMM class around the GEMM with activation class.
GEMM with activation itself is a wrapper around the Blackwell and Hopper default GEMMs.
It supports writing two output tensors---the standard GEMM output (the preactivation) and the standard GEMM output with an activation function such as SwiGLU or ReLU applied (the postactivation).
We define the activation function to be the identity and the postactivation tensor to be the inplace transpose of the preactivation tensor.
Then, when the GEMM with activation class writes to the postactivation, it is really writing to the upper triangle with a transposed layout---this is exactly the intent of the symmetric GEMM kernel.
We override the epilogue of GEMM with activation just to ensure we do not write twice to the diagonal tiles, for the correctness reasons mentioned previously.

\subsection{Kernel Optimizations for Standard Newton-Schulz}
\label{sec:kernel-optimizations-standard-ns}
Using just optimized CuteDSL kernels, we can accelerate standard Newton-Schulz with two changes.

\paragraph{Symmetric Matrix Multiplication}
As discussed above, the matrices $\mA = \mX \mX^\top$ and $\mB = b_t \mA + c_t \mA^2$ computed at each iteration of Newton-Schulz are symmetric by definition.
Therefore, we use our symmetric GEMM kernels for these operations, reducing their FLOP cost by half.

\paragraph{Fused GEMM + Add}
The typical way to implement the non-symmetric multiplication $\mX \gets a_t \mX + \mB \mX$ is to use \texttt{torch.baddbmm}, which calls cuBLAS under the hood.
However, we offer a much faster implementation of this ``Fused GEMM + Add'' operation for Hopper.
Unlike cuBLAS, our ``Fused GEMM + Add'' supports Ping Pong Scheduling for Hopper, which better hides the epilogue addition of $a_t \mX$.

\section{Additional Experiments}
\subsection{Architecture and Optimization Setup}
\label{app:setup}
We now describe details of the experiments in \Cref{sec:exp-nordion2}.
(The benchmarks in \Cref{sec:opt-speedups} are similar, but with minor architectural differences, synthetic data, and GH200s.)

\paragraph{Architecture.}
The models are decoder-only dense transformers (not mixtures of experts).
They use grouped-query attention, sliding-window attention (window $2048$), rotary position embeddings, and RMSNorm, at sequence length $8192$ over a vocabulary of $100{,}352$ tokens.
Weights are stored and updated in MXFP8.
\Cref{tab:app-arch} lists the per-scale dimensions.
Each run uses a single node of eight B200s.

\begin{table}[h]
\centering
\caption{Model dimensions by scale. All share sequence length $8192$, vocabulary $100{,}352$, sliding window $2048$, grouped-query attention, and rotary position embeddings.}
\label{tab:app-arch}
\begin{tabular}{@{}rrrrrr@{}}
\toprule
Scale & Hidden dim & Layers & Heads & KV heads & MLP hidden dim \\
\midrule
1B  & $1536$ & $24$ & $16$ & $8$ & $6656$ \\
3B  & $2304$ & $32$ & $18$ & $6$ & $9856$ \\
4B  & $2560$ & $36$ & $32$ & $8$ & $10752$ \\
7B  & $4096$ & $32$ & $32$ & $8$ & $11008$ \\
14B & $5120$ & $40$ & $40$ & $8$ & $14592$ \\
\bottomrule
\end{tabular}
\end{table}

\paragraph{Optimization.}
Matrix parameters are updated by the orthogonalizing optimizer under study---Muon, NorMuon, or their filtered variants.
For matrix parameters, we multiply the base learning rate by $\sqrt{d_{\text{out}}/d_{\text{in}}}$.
Embeddings, the language-model head, norm scales, and biases are updated by AdamW ($\beta_1 = 0.9$, $\beta_2 = 0.95$) at the base learning rate, without the $\sqrt{d_{\text{out}}/d_{\text{in}}}$ scaling applied to the matrix parameters (the language-model head, in particular, also uses the base rate).
We apply weight decay $0.01$ to the matrix parameters only; the AdamW (non-matrix) groups use no weight decay.
The sequence length is $8192$.
The learning-rate schedule is warmup--stable--cooldown: a $200$-step linear warmup from $2\times10^{-3}$ to the peak learning rate, a constant phase, and a cooldown over the final $25\%$ of training back down to $2\times10^{-3}$.
All reported held-out cross-entropies are taken at the end of cooldown.

\subsection{Finer-Grained Timing Metrics}
\label{app:timing-metrics}
The main text reports GPU (device) time, the metric of interest for the orthogonalization compute.
Here we report three per-step metrics side by side:
\begin{itemize}
\item \textbf{GPU (device) time}, read from CUDA events between two stream markers around \texttt{optimizer.step()};
\item \textbf{CPU (host) time}, from \texttt{time.perf\_counter()} stopped immediately after \texttt{step()} returns and before any device sync --- cost of kernel launch, Python, and any host-side syncs the optimizer performs;
\item \textbf{Communication volume}, the number of bytes moved by the megabatch collectives (\texttt{all\_to\_all}, \texttt{all\_gather}, \texttt{reduce\_scatter}) at each step, measured by tallying the receive-side payload per-GPU.
\end{itemize}
As always, Newton-Schulz runs in half precision arithmetic.
Results for a 14B model sharded over eight B200s and a 7B model on four GH200s appear in \Cref{tab:opt-step-metrics} and \Cref{tab:opt-step-metrics-gh200}, respectively.
Unlike in \Cref{sec:opt-speedups}, we show all four combinations of kernel and Newton-Schulz variant.
We see that Gram Newton-Schulz \emph{without} symmetric kernels achieves up to a $1.2\times$ speedup.
Otherwise, GPU timings are consistent with the results presented in \Cref{sec:opt-speedups}.
The NorMuon family has higher costs overall due to the overhead of its extra normalization steps.
Our contributions still yield similar speedups: up to $1.5\times$ for symmetric kernels and Gram Newton-Schulz, $3.3\times$ for $f=\nicefrac12$, and $5.6\times$ for $f=\nicefrac14$.
\Cref{tab:opt-step-metrics} shows that the CPU cost is negligible---just ${\sim}0.1$ ms across configurations.
This near-constant host cost is a direct consequence of CUDA-graph capture.
Row selection, error feedback, and Gram Newton-Schulz issue many small operations and kernel launches.
Without capture, the host would dispatch each one individually, and this per-launch overhead would dominate, especially for the filtered configurations, whose device work is small.
Capture and replay instead collapse the entire step into a single graph launch, so the host issues one call regardless of how many kernels the step contains or how much they compute.
The reported host time is therefore just the fixed cost of that one launch (${\sim}0.1$ ms), essentially independent of the configuration.
Finally, communication volume is directly proportional to the fraction $f$, as expected.

\begin{table}[ht]
\centering
\caption{Fine-grained per-step metrics for the distributed 14B / 8-GPU benchmark: GPU (device) time from CUDA events, CPU (host) time from \texttt{perf\_counter} stopped before the device sync, and per-GPU communication volume of one step. We distinguish the GPU time between the Muon and NorMuon families of optimizers; CPU time and communication volume are reported once because neither depends on the family: host time is the fixed CUDA-graph launch cost, and communication volume is set by the parameter shapes and the fraction $f$.}
\label{tab:opt-step-metrics}
\small
\setlength{\tabcolsep}{3.5pt}
\begin{tabular}{@{}lllcrrrr@{}}
\toprule
 & & & & \multicolumn{2}{c}{GPU (ms)} & & \\
\cmidrule(lr){5-6}
Configuration & Kernels & Algorithm & NS frac. & Muon & NorMuon & CPU (ms) & Comm (MB) \\
\midrule
Standard Newton-Schulz & PyTorch & Standard & $1$ & 153.8 & 174.1 & 0.1 & 5479 \\
\quad + symmetric kernels & CuteDSL & Standard & $1$ & 132.7 & 153.4 & 0.1 & 5479 \\
\quad + Gram NS & PyTorch & Gram & $1$ & 137.6 & 158.5 & 0.1 & 5479 \\
\quad + kernels + Gram NS & CuteDSL & Gram & $1$ & 106.8 & 127.0 & 0.1 & 5479 \\
\quad + filtering ($f=\nicefrac{1}{2}$) & CuteDSL & Gram & $\nicefrac{1}{2}$ & 66.4 & 73.8 & 0.1 & 2740 \\
\quad + filtering ($f=\nicefrac{1}{4}$) & CuteDSL & Gram & $\nicefrac{1}{4}$ & 34.8 & 38.1 & 0.1 & 1370 \\
\bottomrule
\end{tabular}
\end{table}

\begin{table}[ht]
\centering
\caption{Replication of \Cref{tab:opt-step-metrics} for a 7B-parameter model sharded over a node of four GH200s with NVLink.}
\label{tab:opt-step-metrics-gh200}
\small
\begin{tabular}{@{}lllcrrr@{}}
\toprule
 & & & & \multicolumn{2}{c}{GPU (ms)} & \\
\cmidrule(lr){5-6}
Configuration & Kernels & Algorithm & NS frac. & Muon & NorMuon & Comm (MB) \\
\midrule
Standard Newton-Schulz & PyTorch & Standard & $1$ & 360.6 & 385.4 & 6442 \\
\quad + symmetric kernels & CuteDSL & Standard & $1$ & 274.4 & 315.0 & 6442 \\
\quad + Gram NS & PyTorch & Gram & $1$ & 296.6 & 321.1 & 6442 \\
\quad + kernels + Gram NS & CuteDSL & Gram & $1$ & 217.7 & 252.1 & 6442 \\
\quad + filtering ($f=\nicefrac{1}{2}$) & CuteDSL & Gram & $\nicefrac{1}{2}$ & 101.5 & 116.7 & 3221 \\
\quad + filtering ($f=\nicefrac{1}{4}$) & CuteDSL & Gram & $\nicefrac{1}{4}$ & 59.6 & 68.6 & 1610 \\
\bottomrule
\end{tabular}
\end{table}

\subsection{Benchmarking all-to-all communications}
\label{app:a2a-microbench}

In \Cref{sec:megabatch}, we claim that reducing the number of communication rounds via megabatching is beneficial in two ways: (1) each round adds a fixed overhead not dependent on the size of the payload, and (2) small messages fail to saturate the interconnect.
Here, we verify both claims directly with a standalone microbenchmark of NCCL \texttt{all\_to\_all} over NVLink, independent of any optimizer.
On a single node of four H100s (SXM, NVLink-4 mesh, NCCL 2.29), we sweep the size of the payload per rank from $1$\,KiB to $1$\,GiB. We time each all-to-all with CUDA events (10 warmup, 50 timed iterations), taking the per-iteration time to be the maximum over ranks, and measure the bandwidth.
We confirmed that the transport is a pure NVLink peer-to-peer communication, with no PCIe or network fallback and that the measured bandwidths agree with the \texttt{nccl-tests} \texttt{alltoall\_perf} tool to within $1\%$.

\Cref{fig:a2a-nvlink} shows the results.
Both effects discussed in \Cref{sec:megabatch} are clearly visible.
The latency floor is ${\sim}25\,\mu$s (left panel), of which ${\sim}15$--$17\,\mu$s is host-side dispatch (as measured separately with \texttt{perf\_counter} and no device synchronization).
Without megabatching, we pay this cost over and over.
Bandwidth climbs by more than an order of magnitude as the per-link payload grows (right panel), so coalescing many small all-to-alls into one megabatch increases the effective bandwidth of each transfer.

\begin{figure}[ht]
\centering
\includegraphics[width=\textwidth]{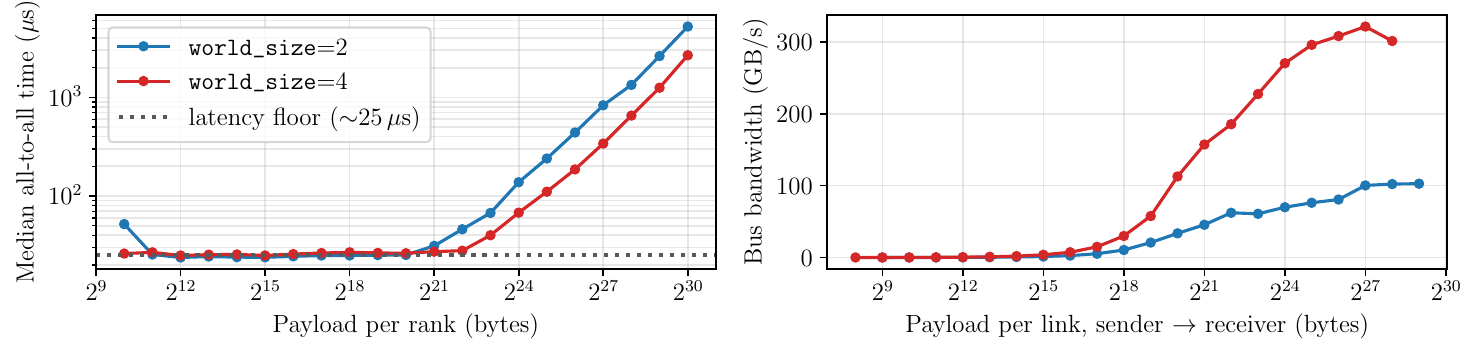}
\caption{NCCL \texttt{all\_to\_all} over NVLink on $4\times$H100.
\textbf{Left:} median all-to-all time versus payload per rank (sender) over 50 trials.
For small payloads, runtime is dominated by a fixed latency.
\textbf{Right:} bus bandwidth versus payload per link (sender $\rightarrow$ receiver).
Bandwidth is near zero for small messages and does not reach $80$--$90\%$ of its peak until the per-link payload is ${\sim}16$--$32$\,MiB.
At $\texttt{world\_size} = 4$, a $256$\,KiB per-link payload achieves only ${\sim}10\%$ of the peak bandwidth.}
\label{fig:a2a-nvlink}
\end{figure}

\subsection{Ablations}
\label{app:nordion2-ablations}
\paragraph{Tuning the baseline}
The experiments in \Cref{sec:exp-nordion2} compare Dion3 to a NorMuon baseline.
\Cref{tab:normuon-tuning}, below, demonstrates that this baseline is properly tuned for a fair comparison.

\begin{table}[ht]
\centering
\caption{Tuning hyperparameters (learning rate $\eta$, momentum $\mu$) of NorMuon at 1B.
Final validation loss on 100B tokens of ClimbMix.
Winner: $\eta=0.01$, $\mu=0.9$ (bold).}
\label{tab:normuon-tuning}
\begin{tabular}{@{}cc@{}}
\toprule
\multicolumn{2}{@{}c@{}}{Sweep $\eta$ (at $\mu=0.9$)} \\
\cmidrule(lr){1-2}
$\eta$ & CE \\
\midrule
$0.005$ & $2.220$ \\
$0.01$  & $\mathbf{2.194}$ \\
$0.02$  & $2.215$ \\
$0.03$  & $2.237$ \\
\bottomrule
\end{tabular}
\qquad\qquad
\begin{tabular}{@{}cc@{}}
\toprule
\multicolumn{2}{@{}c@{}}{Sweep $\mu$ (at $\eta=0.01$)} \\
\cmidrule(lr){1-2}
$\mu$ & CE \\
\midrule
$0.9$   & $\mathbf{2.194}$ \\
$0.949$ & $2.201$ \\
$0.974$ & $2.217$ \\
$0.987$ & $2.238$ \\
\bottomrule
\end{tabular}
\end{table}

\paragraph{Dion3 with no compression}
As a control, we run Dion3 with $f = 1$, which should reduce to plain NorMuon.
\Cref{fig:nordion2-loss-curve-1b} confirms this at 1B; the two curves coincide throughout training, and their final validation losses differ by $0.0005$.

\begin{figure}[ht]
\centering
\includegraphics[width=\textwidth]{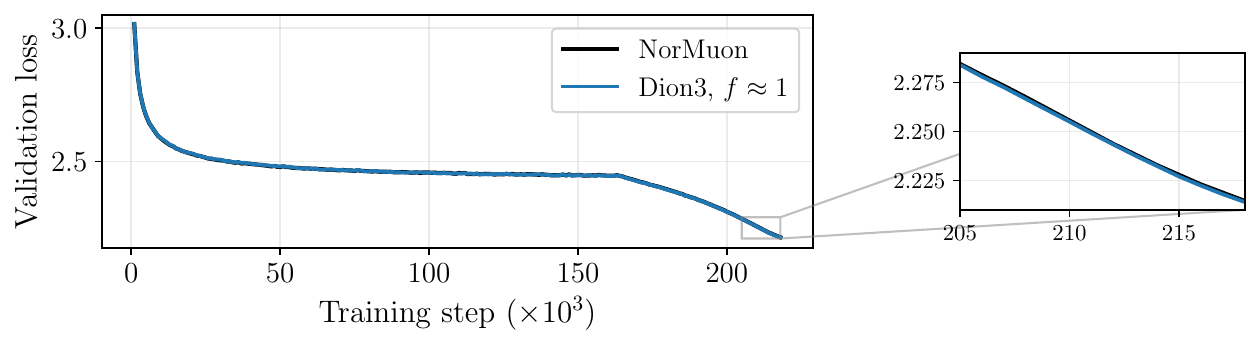}
\caption{Validation loss at 1B for NorMuon and for Dion3 at $f \approx 1$, on 100B tokens of ClimbMix.
As expected, Dion3 tracks NorMuon throughout, ending at $2.2141$ versus $2.2146$.}
\label{fig:nordion2-loss-curve-1b}
\end{figure}

\paragraph{Timings with NorMuon family}
\Cref{sec:opt-speedups} compares the runtime of Muon to the Muon version of Dion3.
Here, in \Cref{fig:normuon-optstep-relative-muon}, we repeat this benchmark with the NorMuon family.
Results are quite alike, but due to the extra overhead of NorMuon's normalization steps, it benefits slightly less from our contributions, which target the orthogonalization step.

\begin{figure}[ht]
\centering
\includegraphics[width=\textwidth]{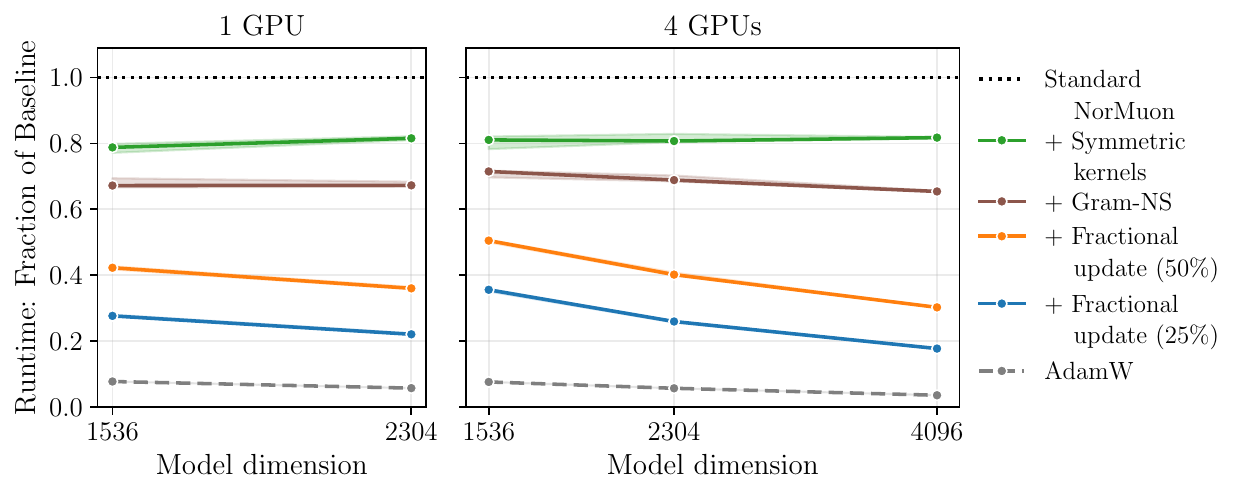}
\caption{Optimizer step time (excluding forward/backward pass) relative to standard Muon across model scales when training on 1 GH200 (left) and FSDP over 4 GH200s (right).
Each colored line adds one of our contributions on top of the previous one; AdamW is shown for reference.
Lines show median over 25 steps and bands show interquartile range.
Compare to \Cref{fig:optstep-relative-muon}.}
\label{fig:normuon-optstep-relative-muon}
\end{figure}

\section{Case Studies of End-to-End Training Time}
\label{app:end-to-end_time}
The share of end-to-end training time taken up by Newton-Schulz can vary widely depending on the training setup.
To explain this variability, we analyze two idealized scenarios.
In one, standard Newton-Schulz takes 2\% of training time; in the other it takes 17\%.

\subsection*{Case Study 1: Standard Newton-Schulz takes 2\% of Kimi K2 training time}

The following analysis gives a very optimistic estimate of the optimizer's wall clock time.
We assume an efficient training infrastructure with highly optimized pipeline parallelism.
Moreover, we assume that the optimizer step of each pipeline stage is completely hidden behind the backward pass of the next pipeline stage.

Kimi K2 Thinking is a $1.1$ trillion parameter model with $32$ billion active parameters.
It has $1$ dense layer followed by $60$ MoE layers \citep{kimiteam2026kimik2openagentic}.
It is pretrained with $256$-GPU model parallel groups, $16$-way pipeline parallelism, $16$-way expert parallelism within each pipeline stage, and a huge batch size of $67$ million tokens.

We use a single H100 to approximate the share of each training step's runtime occupied by Newton-Schulz in this setting under the following assumptions:
\begin{enumerate}
  \item The training cluster consists of $256$ nodes of eight H100s each ($2048$ GPUs in total), connected with NDR 400 Gb/s InfiniBand inter-node (8 NICs per node, 1:1 NIC-to-GPU ratio) and NVLink 4.0 intra-node.
  This is the size of the cluster used to train DeepSeek-V3, with upgraded hardware \citep{liu2024deepseek}.
  \item Training in \texttt{bfloat16} attains $40\%$ model flop utilization (MFU), which is typical for MoEs at this scale on H100s.
  \item The only non-overlapped optimizer time is that of the last pipeline stage to complete its backward pass (i.e.,\ pipeline stage $1$ of $16$).
  The optimizer steps of pipeline stages $2$ to $16$ are fully hidden behind the backwards of stages $1$ to $15$.
  \item Pipeline stage $1$ contains the dense layer and $3$ MoE layers.
\end{enumerate}

Under these assumptions, the optimal way to partition the Newton-Schulz work of pipeline stage $1$ is as follows.
Each of the $16$ GPUs in pipeline stage $1$'s expert parallel group gets
\begin{equation*}
\frac{384 \text{ experts/layer} \times 3 \text{ MoE layers}}{16 \text{ GPUs}}
= 72 \text{ experts/GPU}
= 216 \text{ expert up-gate-down/GPU}.
\end{equation*}
Each of the $16$ GPUs has its own unique expert weights, so no communication is needed.
Pipeline stage $1$ also contains the dense MLP's three $7168\times18432$ weights (up/gate/down) and three shared experts.
Orthogonalizing the dense MLP weights is the dominant cost, so they are sent to three different GPUs; the shared experts are split amongst the remaining $13$ GPUs.
Thus, the orthogonalization time for pipeline stage $1$ is the time it takes for a single GPU to orthogonalize $216$ expert up/gate/down weights and $1$ dense up/gate/down weight.
As benchmarked in \Cref{fig:kimi-k2-ns}, standard Newton-Schulz implemented in PyTorch takes $315$ ms to do this.
Per our assumption, stage $1$'s orthogonalization step is the only one that is not overlapped.

Having estimated the Newton-Schulz time, we now estimate the end-to-end wall clock time of an entire Kimi K2 global training step.
We use the standard estimate of $6NB$ FLOPs for the forward and backward pass, where $N$ is the number of active parameters and $B$ is the global batch size.
Given
\begin{itemize}
  \item Active parameters: $N = 32 \cdot 10^9$
  \item H100 peak: $P = 989 \cdot 10^{12}$ FLOP/s
  \item Model flop utilization: $\text{MFU} = 40\%$
  \item Cluster size: $G = 2048$ GPUs
  \item Global batch size: $B = 67 \cdot 10^6$ tokens
\end{itemize}
we have
\begin{equation*}
  \text{sec/batch}
  = \frac{6N \times B}{P \times \text{MFU} \times G}
  = 15.9
\end{equation*}

Thus, Newton-Schulz takes approximately $\frac{315\text{ ms}}{15900\text{ ms} + 315\text{ ms}} = 1.9\%$ of total pretraining wall clock time in this setting.

\subsection*{Case Study 2: Standard Newton-Schulz takes 17\% of Llama3-70B SFT time}

Llama3-70B is an 80-layer dense model with hidden size $8192$, intermediate size $28672$, and grouped query attention with $\mW_K, \mW_V$ of size $1024 \times 8192$ and $\mW_Q, \mW_O$ of size $8192 \times 8192$ \citep{grattafiori2024llama3herdmodels}.
Supervised finetuning (SFT) typically uses small batch sizes, ranging from $32$ to $256$ sequences \citep{liu2024deepseek}.

We analyze the following SFT case:
\begin{enumerate}
  \item Training uses $32$ H100s across $4$ nodes (8 GPUs per node).
  \item Training in \texttt{bfloat16} hits $40\%$ MFU.
  \item Weights are sharded evenly across GPUs using FSDP, and the exposed Newton-Schulz time is that of $80 \text{ layers} / 32 \text{ GPUs} \approx 3 \text{ layers}$.
  Each layer has $3$ MLP weights (up, gate, down), and the attention weights $\mW_Q$, $\mW_K$, $\mW_V$, and $\mW_O$.
\end{enumerate}

According to our benchmarking, standard Newton-Schulz of
\begin{itemize}
  \item nine $8192 \times 28672$ weights takes $739$ ms,
  \item six $8192 \times 8192$ weights takes $156$ ms,
  \item six $1024 \times 8192$ weights takes $2.32$ ms,
\end{itemize}
totaling $897$ ms. Given
\begin{itemize}
  \item Parameters: $N = 70 \cdot 10^9$
  \item H100 peak: $P = 989 \cdot 10^{12}$ FLOP/s
  \item Model flop utilization: $\text{MFU} = 40\%$
  \item Cluster size: $G = 32$ GPUs
  \item Global batch size: $B = 64 \text{ sequences} \times 2048 \text{ tokens/sequence} = 131{,}072$ tokens
\end{itemize}
we have
\begin{equation*}
  \text{sec/batch}
  = \frac{6N \times B}{P \times \text{MFU} \times G}
  = 4.35\text{ s}
\end{equation*}

Thus, Newton-Schulz takes approximately $\frac{897\text{ ms}}{4350\text{ ms} + 897\text{ ms}} = 17\%$ of total SFT wall clock time in this setting.

\end{appendices}

\end{document}